\pdfoutput=1 %

\documentclass[onefignum,onetabnum]{siamonline220329}

\usepackage{tikz-cd}

\usepackage{crossreftools}
\usepackage{enumitem}
\setlist[enumerate]{leftmargin=.5in}
\setlist[itemize]{leftmargin=.5in}

\newsiamremark{remark}{Remark}
\newsiamremark{hypothesis}{Hypothesis}
\crefname{hypothesis}{Hypothesis}{Hypotheses}
\newsiamthm{claim}{Claim}

\headers{Function-Space Optimality of Neural Architectures}{R. Parhi and M. Unser}

\title{Function-Space Optimality of Neural Architectures with Multivariate Nonlinearities\thanks{\vspace{-1em}
\funding{This work was supported by the Swiss National Science Foundation under Grant 200020\_219356 / 1.}}}

\author{Rahul Parhi\thanks{Department of Electrical and Computer Engineering, University of California, San Diego, La Jolla, CA 92093, USA (\email{rahul@ucsd.edu}). Part of this work was done while the author was with the Biomedical Imaging Group, \'Ecole polytechnique f\'ed\'erale de Lausanne, CH-1015 Lausanne, Switzerland.} \and Michael Unser\thanks{Biomedical Imaging Group, \'Ecole polytechnique f\'ed\'erale de Lausanne, CH-1015 Lausanne, Switzerland
(\email{michael.unser@epfl.ch}).}}

\usepackage{macros}

\begin{document}
\maketitle

\begin{abstract}
    We investigate the function-space optimality (specifically, the Banach-space optimality) of a large class of shallow neural architectures with multivariate nonlinearities/activation functions. To that end, we construct a new family of Banach spaces defined via a regularization operator, the $k$-plane transform, and a sparsity-promoting norm. We prove a representer theorem that states that the solution sets to learning problems posed over these Banach spaces are completely characterized by neural architectures with multivariate nonlinearities. These optimal architectures have skip connections and are tightly connected to orthogonal weight normalization and multi-index models, both of which have received recent interest in the neural network community. Our framework is compatible with a number of classical nonlinearities including the rectified linear unit (ReLU) activation function, the norm activation function, and the radial basis functions found in the theory of thin-plate/polyharmonic splines. We also show that the underlying spaces are special instances of reproducing kernel Banach spaces and variation spaces. Our results shed light on the regularity of functions learned by neural networks trained on data, particularly with multivariate nonlinearities, and provide new theoretical motivation for several architectural choices found in practice.

\end{abstract}



\section{Introduction}
In supervised machine learning, the goal is to predict an output $y \in \Y$ (e.g., a label or response) from an input $\vec{x} \in \X$ (e.g., a feature or example), where $\X$ and $\Y$ denote the domain of the inputs and outputs, respectively. One solves this task by ``training'' a model to fit a set of data which consists of a finite number of input-output pairs $\curly{(\vec{x}_m, y_m)}_{m=1}^M \subset \X \times \Y$. The goal is to ``learn'' a function $f: \X \to \Y$ with $f(\vec{x}_m) \approx y_m$, $m = 1, \ldots, M$, such that $f$ can accurately predict the output $y \in \Y$ of a new input $\vec{x} \in \X$. This task is usually formulated as an optimization problem of the form
\begin{equation}
    \min_{f \in \F} \: \sum_{m=1}^M \Loss(y_m, f(\vec{x}_m)) + \lambda \Phi(f),
    \label{eq:learning-problem}
\end{equation}
where $\F$ is a prescribed \emph{model class} of functions that map $\X \to \Y$, $\Loss(\dummy, \dummy)$ is a \emph{loss function}, and $\Phi: \F \to \R_{\geq 0}$ is a \emph{regularization functional} that injects prior knowledge/regularity on the function to be learned. The hyperparameter $\lambda > 0$ controls the tradeoff between data fidelity and regularity. Without the inclusion of the regularization functional in \cref{eq:learning-problem}, the problem is typically ill-posed. Indeed, in many practical scenarios the problem is \emph{overparameterized} as the dimension of the model class $\F$ greatly exceeds the number $M$ of data. A classical choice of model class is a reproducing kernel Hilbert space (RKHS). The accompanying regularization functional is the squared Hilbert norm of the RKHS. In this scenario, the RKHS representer theorem establishes that there exists a solution to \cref{eq:learning-problem} that takes the form of a linear combination of reproducing kernels centered at the data sites~\cite{deBoorSmoothingSplines,WahbaSmoothingSplines3,ScholkopfKernels,wahba1990spline,WendlandBook}. This provides an exact characterization of the function-space optimality of kernel methods.

Recently, there has been a line of work that investigates the function-space optimality of neural networks~\cite{Bach,BartolucciRKBS,OngieRadon,parhi2020role,ParhiShallowRepresenter,ParhiDeepRepresenter,ParhiMinimax,SavareseInfiniteWidth,SpekRKBS,unser2022kernel}. Crucially, these works define and study (non-Hilbertian) Banach spaces defined by sparsity or variation. These spaces have an analytic description via the Radon transform~\cite{KurkovaEstimates,OngieRadon,ParhiShallowRepresenter}. The accompanying neural network representer theorems for these spaces were first established in~\cite{ParhiShallowRepresenter} and then studied and refined by a number of authors~\cite{BartolucciRKBS,ParhiDeepRepresenter,SpekRKBS,unser2022kernel}. While these results characterize the function-space optimality of neural networks, they only consider univariate nonlinearities. We refer the reader to the recent survey~\cite{ParhiSPM} for an up-to-date summary of this research direction. The purpose of this paper is to further extend the existing results on the function-space optimality of neural architectures, with a particular focus on \emph{multivariate nonlinearities}, which have gained recent interest in the neural network community~\cite{anil2019sorting,goodfellow2013maxout,gulcehre2014learned,mhaskar2023approximation}.

The form of a neuron with an $m$-variate nonlinearity, $1 \leq m \leq d$, is
\begin{equation}
    \vec{x} \mapsto \rho(\mat{A}\vec{x} - \vec{t}), \quad \vec{x} \in \R^d,
    \label{eq:atom}
\end{equation}
where $\rho: \R^m \to \R$ is the nonlinearity (or activation function), $\mat{A} \in \R^{m \times d}$ is a weight matrix that controls the orientation of the neuron, and $\vec{t} \in \R^m$ is a bias which controls the offset of neuron. When $m = 1$,  these atoms can be written as
\begin{equation}
    \vec{x} \mapsto \rho(\vec{\alpha}^\T\vec{x} - t), \quad \vec{x} \in \R^d,
\end{equation}
with $\rho: \R \to \R$, $\vec{\alpha} \in \R^d$, and $t \in \R$, which is the form of a classical neuron with a univariate nonlinearity. Neurons with $m$-variate nonlinearities as in \cref{eq:atom} have been studied under many different names including \emph{$m$-sparse functions}~\cite{abbe2022merged,ghorbani2020neural}, \emph{generalized ridge functions}~\cite{keiper2019approximation}, \emph{$(d - m)$-plane ridge functions}~\cite{Parhikplane}, and \emph{multi-index models}~\cite{cohen2012capturing,dalalyan2008new,fukumizu2004dimensionality,li1991sliced,liu2024learning}. Notably, multi-index models have gained recent interest from the neural network community~\cite{abbe2022merged,Bach,ghorbani2020neural,parkinson2023linear}. A shallow neural architecture with an $m$-variate nonlinearity $\rho: \R^m \to \R$ takes the form
\begin{equation}
    \vec{x} \mapsto \sum_{n=1}^N v_n \, \rho(\mat{A}_n\vec{x} - \vec{t}_n), \quad \vec{x} \in \R^d,
    \label{eq:translation-network}
\end{equation}
where, for $n = 1, \ldots, N$, $v_n \in \R$, $\mat{A}_n \in \R^{m \times d}$, and $\vec{t}_n \in \R^m$. Such architectures are sometimes called \emph{generalized translation networks}~\cite{mhaskar2023approximation,MHASKAR1992350,mhaskar1995degree} and are classically known to be universal approximators if and only if\footnote{This equivalence holds under the global assumption that $\rho: \R^m \to \R$ does not grow faster than a polynomial, i.e., it is a \emph{tempered} function.} $\rho: \R^m \to \R$ is not a polynomial~\cite[Corollary~3.3]{MHASKAR1992350}. These architectures have also recently been studied in the context of ridgelet analysis for a variety of shallow neural architectures~\cite{sonoda2024unified}. In this paper, we characterize, for all integers $m$ with $1 \leq m \leq d$, the function-space optimality of neural architectures with $m$-variate nonlinearities of the form \cref{eq:translation-network}, for a large class of nonlinearities. We show that these architectures are optimal solutions to data-fitting problems posed over (non-Hilbertian) Banach spaces defined via a sparsity-promoting norm in the domain of the $k$-plane transform. When $m = 1$, our framework is compatible with univariate nonlinearities and classical neural architectures, including the ReLU activation function. At the opposite extreme ($m = d$) our framework encompasses sparse kernel expansions and radial basis functions~\cite{AziznejadSparse,rosset2007l1,steinwart2003sparseness}. To the best of our knowledge, the results for $1 < m < d$ are new.

\subsection{Main Contributions and Road Map}

Our results shed light on the regularity of the functions learned by neural networks trained on data. They provide new theoretical motivation for several architectural choices often found in practice, particularly with multivariate nonlinearities. These results hinge on recent developments regarding the distributional extension and invertibility of the $k$-plane transform and its dual~\cite{Parhikplane}. The main contributions and organization of this paper are summarized in the remainder of this section.

\paragraph{New Neural Network Banach Spaces} We propose and study the properties of a new family of \emph{native spaces}, defined by
    \begin{equation}
        \M_{\LOp}^k(\R^d) \coloneqq \curly*{f: \R^d \to \R \text{ is measurable\footnotemark} \st \begin{aligned}
        &\norm{\KOp_{d-k}\RadonOp_k\LOp f}_\M < \infty, \\ 
        &\esssup_{\vec{x} \in \R^d} \: \abs{f(\vec{x})} (1 + \norm{\vec{x}}_2)^{-n_{\LOp}} < \infty
        \end{aligned}} \subset \Sch'(\R^d),
        \label{eq:native-space-intro}
    \end{equation}
    \footnotetext{We refer to a function as \emph{measurable} when it is measurable with respect to the Lebesgue $\sigma$-algebra.}%
    where $k$ is an integer such that $0 \leq k < d$, $\Sch'(\R^d)$ denotes the
    space of tempered distributions, $\LOp$ is a \emph{$k$-plane-admissible}
    pseudodifferential operator (in the sense of \cref{defn:admissible}),
    $\RadonOp_k$ denotes the $k$-plane transform, and $\KOp_{d-k}$ is the
    filtering operator of computed tomography which is such that $\RadonOp_k^*
    \KOp_{d-k} \RadonOp_k = \Id$. The $\M$-norm denotes the total variation norm
    (in the sense of measures). It can be viewed as a ``generalization'' of the $L^1$-norm that can also be applied to distributions such as the Dirac impulse. Said differently, if $f \in \M_{\LOp}^k(\R^d)$ is such that $\KOp_{d-k}\RadonOp_k\LOp f$ is a \emph{bona fide} function (not a distribution), then
    \begin{equation}
        \norm{\KOp_{d-k}\RadonOp_k\LOp f}_\M = \norm{\KOp_{d-k}\RadonOp_k\LOp f}_{L^1}.
    \end{equation}
    The growth restriction of degree $n_{\LOp}$
    plays the role of a proxy to the order of $\LOp$; more specifically,
    $n_{\LOp}$ is the highest polynomial degree annihilated by $\LOp$. The growth
    restriction in the definition of the native space ensures that the null
    space of the operator $\KOp_{d-k}\RadonOp_k\LOp$ is finite-dimensional. In \cref{sec:native-spaces}, we prove that, when equipped with an appropriate direct-sum topology, $\M_{\LOp}^k(\R^d)$ forms a Banach space that is isometrically isomorphic to the Cartesian product of a space of (Radon) measures with the space of polynomials of degree at most $n_{\LOp}$. They add to the growing list of ``neural Banach spaces'' that are currently being actively investigated~\cite{siegel2023characterization}.
    
    \paragraph{Representer Theorems for Neural Networks with Multivariate Nonlinearities} We prove a representer theorem (\cref{thm:representer}) that states that, under mild assumptions on the loss function and $\LOp$, the solution set to the optimization
      problem
    \begin{equation}
        \min_{f \in \M_{\LOp}^k(\R^d)} \: \sum_{m=1}^M \Loss(y_m, f(\vec{x}_m)) + \lambda \norm{\KOp_{d-k}\RadonOp_k\LOp f}_\M
        \label{eq:variational-problem-intro}
    \end{equation}
    is completely characterized by shallow neural architectures with
    $(d-k)$-variate activation functions matched to the operator $\LOp$ and
    widths bounded by the number $M$ of data (independent of the dimension $d$ of
    the data). This result sheds light on the role of biases, skip
    connections, and the use of structured weight matrices in neural
    architectures. Indeed, these architectures take the form
    \begin{equation}
      \vec{x} \mapsto c(\vec{x}) + \sum_{n=1}^N v_n \, \rho_{\LOp}(\mat{A}_n\vec{x} -
      \vec{t}_n),
      \label{eq:solution-form-intro}
    \end{equation}
    with $N \leq M$, where, for $n = 1, \ldots, N$, $\mat{A}_n \in \R^{(d-k) \times d}$ is such that $\mat{A}_n\mat{A}_n^\T = \mat{I}_{d-k}$ (identity matrix), $\vec{t}_n \in \R^{d-k}$, and $v_n \in \R \setminus \curly{0}$. The function $c$ is a polynomial of degree at most $n_{\LOp}$ and the function $\rho_{\LOp}: \R^{d-k} \to \R$ is a $(d-k)$-variate nonlinearity matched to the operator $\LOp$. Finally, the regularization cost of \cref{eq:solution-form-intro} is $\sum_{n=1}^N \abs{v_n} = \norm{\vec{v}}_1$.
    The term $\vec{x} \mapsto c(\vec{x})$ that appears in \cref{eq:solution-form-intro} can be viewed as a (generalized) skip connection in neural network parlance. Note that \cref{eq:solution-form-intro} is exactly a sparse combination of multi-index models with learnable orientations $\mat{A}_n$, offsets $\vec{t}_n$, and fixed profiles specified by the multivariate nonlinearity $\rho_{\LOp}$ as well as a generalized translation network as in \cref{eq:translation-network}. Thus, if the data lied on a low-dimensional subspace (or union of subspaces), the neural architecture could automatically adapt to this structure and avoid the curse of dimensionality. 

    \paragraph{Connections to Prior Work}
    In \cref{sec:examples}, we instantiate our results on the function-space optimality of neural architectures. First, we discuss implications of our representer theorem to the training and regularization of neural networks. These results provide new insight into the role of overparameterization and the use of \emph{orthogonal weight normalization} in network architectures, which corresponds to the property that $\mat{A}_n\mat{A}_n^\T = \mat{I}_{d-k}$ in \cref{eq:solution-form-intro}~\cite{anil2019sorting,huang2018orthogonal,huang2023normalization,li2019preventing}. This property has been shown to increase the stability~\cite{anil2019sorting} and generalization properties~\cite{huang2023normalization} of neural architectures. We then discuss specific examples of neural architectures that are compatible with our framework. These architectures include the popular ReLU~\cite{relu-sparse}, the norm activation function/nonlinearity~\cite{gulcehre2014learned}, and the radial basis functions found in the theory of thin-plate/polyharmonic splines~\cite{duchon1977splines,WendlandBook}. In particular, our theory provides a way to interpolate between the completely anisotropic atoms found in neural architectures with univariate nonlinearities ($k = (d-1)$) to the completely isotropic atoms found in the theory of sparse kernel expansions and radial basis functions ($k = 0$) in a similar vein to how $\alpha$-molecules interpolate between ridgelets (anisotropic) and wavelets (isotropic)~\cite{grohs2016alpha}.
    
    In \cref{sec:RKBS-variation}, we discuss how the native space $\M_{\LOp}^k(\R^d)$ can be viewed as an example of a reproducing kernel Banach space (RKBS)~\cite{BartolucciRKBS,lin2022reproducing,SpekRKBS,zhang2009reproducing} as well as an example of a variation space~\cite{Bach,devore2023weighted,kurkova2001bounds,mhaskar2004tractability,shenouda2023vector,SXSharp,siegel2023characterization}. These are classical approaches used for the understanding of neural networks through the lens of functional analysis and approximation theory. Thus, any abstract result for RKBSs or variation spaces immediately applies to $\M_{\LOp}^k(\R^d)$.

\section{Mathematical Preliminaries and Notation}

The Schwartz space of smooth and rapidly decreasing functions on $\R^d$ is denoted by $\Sch(\R^d)$. Its continuous dual is the space $\Sch'(\R^d)$ of tempered distributions. We let $L^p(\R^d)$ denote the Lebesgue space for $1 \leq p \leq \infty$ and define the weighted $L^\infty$-space
\begin{equation}
  L^\infty_{-\alpha}(\R^d) \coloneqq \curly*{f: \R^d \to \R \text{ is measurable} \st \norm{f}_{L^\infty_{-\alpha}} \coloneqq \esssup_{\vec{x} \in \R^d} \: \abs{f(\vec{x})} (1 + \norm{\vec{x}}_2)^{-\alpha} < \infty}.
\end{equation}
This is the space of growth-restricted functions with rate $\alpha \in \R$. It is a Banach space that can be identified as the continuous dual of the weighted $L^1$-space
\begin{equation}
    L^1_{\alpha}(\R^d) \coloneqq \curly*{f: \R^d \to \R \text{ is measurable} \st \norm{f}_{L^1_{\alpha}} \coloneqq \int_{\R^d} \abs{f(\vec{x})} (1 + \norm{\vec{x}}_2)^{\alpha} \dd\vec{x} < \infty}.
\end{equation}

The Banach space of continuous functions vanishing at $\pm \infty$ on $\R^d$ equipped with the $L^\infty$-norm is denoted by $C_0(\R^d)$. By the Riesz--Markov--Kakutani representation theorem~\cite[Chapter~7]{FollandRA}, its continuous dual can be identified with the Banach space of finite Radon measures, denoted $\M(\R^d)$. Since $\Sch(\R^d)$ is dense in $C_0(\R^d)$, we have, by duality, that $\M(\R^d)$ is continuously embedded in $\Sch'(\R^d)$.  Given a space $\X$ and a norm $\norm{\dummy}$, the completion of $\X$ in $\norm{\dummy}$ is a Banach space, denoted by $\cl{(\X, \norm{\dummy})}$. For example, we have, for $1\leq p < \infty$, that $L^p(\R^d) = \cl{(\Sch(\R^d), \norm{\dummy}_{L^p})}$, and $C_0(\R^d) = \cl{(\Sch(\R^d), \norm{\dummy}_{L^\infty})}$.

The Fourier transform of $\varphi \in \Sch(\R^d)$ is defined as
\begin{equation}
    \hat{\varphi}(\vec{\xi}) \coloneqq \FourierOp\curly{\varphi}(\vec{\xi}) = \int_{\R^d} \varphi(\vec{x}) e^{-\imag \vec{\xi}^\T\vec{x}} \dd\vec{x}, \quad \vec{\xi} \in \R^d,
\end{equation}
where $\imag^2 = -1$. Consequently, the inverse Fourier transform of $\hat{\varphi} \in \Sch(\R^d)$ is given by
\begin{equation}
    \FourierOp^{-1}\curly{\hat{\varphi}}(\vec{x}) = \frac{1}{(2\pi)^d} \int_{\R^d} \hat{\varphi}(\vec{\xi}) e^{\imag \vec{\xi}^\T\vec{x}} \dd\vec{\xi}, \quad \vec{x} \in \R^d.
\end{equation}
These operators are extended to act on $\Sch'(\R^d)$ by duality.

Any continuous linear shift-invariant (LSI) operator $\LOp: \Sch(\R^d) \to \Sch'(\R^d)$ is a convolution operator specified by a unique kernel $h \in \Sch'(\R^d)$ such that $\LOp \varphi = h * \varphi$. Such operators can also be specified in the Fourier domain by
\begin{equation}
    \LOp \varphi = \FourierOp^{-1}\curly{\hat{L} \hat{\varphi}},
\end{equation}
where $\hat{L} \in \Sch'(\R^d)$ is the Fourier transform of the kernel $h \in \Sch'(\R^d)$. The tempered distribution $h$ is the \emph{impulse response} of $\LOp$ and the tempered distribution $\hat{L}$ is the Fourier symbol or \emph{frequency response} of $\LOp$. We shall generally use upright, roman letters for LSI operators and use the italic variant with a hat to denote its frequency response. 

\subsection{The \texorpdfstring{$k$}{k}-Plane Transform}
We are going to adopt the parameterization of the $k$-plane transform from~\cite[Section~4]{Parhikplane}. There, the space of $k$-planes is parameterized by the Cartesian product of the Stiefel manifold with $\R^{d-k}$, where $k$ is an integer such that $1 \leq k < d$. Let
\begin{equation}
    V_{d-k}(\R^d) \coloneqq \curly{\mat{A} \in \R^{(d-k) \times d} \st \mat{A} \mat{A}^\T = \mat{I}_{d-k}}
\end{equation}
denote the Stiefel manifold. Then, the $k$-plane transform of $\varphi \in \Sch(\R^d)$ is defined as
\begin{equation}
    \RadonOp_k\curly{\varphi}(\mat{A}, \vec{t}) = \int_{\R^d} \varphi(\vec{x}) \delta(\mat{A}\vec{x} - \vec{t}) \dd\vec{x}, \quad (\mat{A}, \vec{t}) \in (V_{d-k}(\R^d), \R^{d-k}),
\end{equation}
where $\delta \in \Sch'(\R^{d-k})$ is the $(d-k)$-variate Dirac impulse\footnote{The distribution $\delta \in \Sch'(\R^{d-k})$ is such that $\ang{\delta, \phi} = \phi(\vec{0})$ for all $\phi \in \Sch(\R^{d-k})$.} and the integral is understood as the action of $\delta(\mat{A}(\dummy) - \vec{t}) \in \Sch'(\R^d)$ on $\varphi \in \Sch(\R^d)$. Further, the dual transform (often called the ``backprojection'') of $g \in L^\infty(V_{d-k}(\R^d) \times \R^{d-k})$ is given by
\begin{equation}
  \RadonOp_k^*\curly{g}(\vec{x}) = \int_{V_{d-k}(\R^d)} g(\mat{A},
  \mat{A}\vec{x}) \dd\mat{A}, \quad \vec{x} \in \R^d,
\end{equation}
where $\dd\mat{A}$ denotes integration against the Haar measure of $V_{d-k}(\R^d)$. Since we impose that the rows of $\mat{A}$ are orthonormal, we have that $(\mat{A}, \vec{t})$ and $(\mat{U}\mat{A}, \mat{U}\vec{t})$ define the same $k$-plane, for any orthogonal transformation $\mat{U} \in \mathrm{O}_{d-k}(\R)$ (the orthogonal group in dimension $(d-k)$). The main advantage of the proposed parameterization is that it will allow us to identify the symmetries of $k$-plane domain as ``isotropic'' symmetries.

Letting $\Xi_k \coloneqq V_{d-k}(\R^d) \times \R^{d-k}$ denote the $k$-plane domain, we define the space of Schwartz functions on $\Xi_k$, denoted by $\Sch(\Xi_k)$, as the space of smooth functions that are rapidly decreasing in the $\vec{t} \in \R^{d-k}$ variable~\cite{GonzalezRange}. More specifically, we have that $\Sch(\Xi_k) = C^\infty(V_{d-k}(\R^d)) \,\hat{\otimes}\, \Sch(\R^{d-k})$, where $\hat{\otimes}$ denotes the topological tensor product, which is the completion of the algebraic tensor product with respect to the projective topology~\cite[Chapter~43]{TrevesTVS}. We state in \cref{prop:cont-inv} a classical result regarding the continuity and invertibility of the $k$-plane transform.
\begin{proposition}[{see~\cite{GelfandIntegralGeometry,GonzalezRange,KeinertInversion,Parhikplane,RubinInversion,SmithRadiographs,SolmonXRay}}] \label[proposition]{prop:cont-inv}
  The operator $\RadonOp_k$ continuously maps $\Sch(\R^d)$ into $\Sch(\Xi_k)$.
  Moreover,
  \begin{equation}
    \RadonOp_k^* \KOp_{d-k} \RadonOp_k = c_{d,k} (-\Delta)^\frac{k}{2} \RadonOp_k^*
    \RadonOp_k = c_{d,k} \RadonOp_k^* \RadonOp_k (-\Delta)^{\frac{k}{2}} = \Id
  \end{equation}
  on $\Sch(\R^d)$, with
  \begin{equation}
    c_{d,k} = \frac{1}{(2\pi)^k} \frac{\abs{\Sph^{k-1}}}{\abs{\Sph^{d-k-1}}} \frac{1}{\prod_{n=k}^{d-1} \abs{\Sph^{n-1}}},
  \end{equation}
  where $\abs{\dummy}$ denotes the surface area.
  The underlying operators are the $d$-variate Laplacian operator $\Delta$ and the filtering operator\footnote{In computed tomography (CT), this filter is referred to as the backprojection filter found in the filtered backprojection algorithm for CT image reconstruction.} $\KOp_{d-k} = c_{d,k} (-\Delta_{d-k})^{k/2}$, where $\Delta_{d-k}$ denotes the $(d-k)$-variate Laplacian applied to the $\vec{t} \in \R^{d-k}$ variable. The filtering operator is equivalently specified by the frequency response $\hat{K}_{d-k}(\vec{\omega}) = c_{d,k} \norm{\vec{\omega}}_2^k$, $\vec{\omega} \in \R^{d-k}$.
\end{proposition}

The $k$-plane transform has tight connections with the Fourier transform. This is summarized in the so-called \emph{Fourier slice theorem}.
\begin{proposition}[{\cite[Corollary~7.5]{Parhikplane}}]
\label[proposition]{thm:Fourier-slice}
    Given $\varphi \in \Sch(\R^d)$, we have that
    \begin{equation}
        \reallywidehat{\RadonOp_k\curly{\varphi}(\mat{A}, \dummy)}(\vec{\omega}) = \hat{\varphi}(\mat{A}^\T\vec{\omega}), \quad \vec{\omega} \in \R^{d-k}, \mat{A} \in V_{d-k}(\R^d),
        \label{eq:Fourier-slice-S}
    \end{equation}
    where the Fourier transform on the left-hand side is the $(d - k)$-variate transform and the Fourier transform on the right-hand side is the $d$-variate Fourier transform.
\end{proposition}
\begin{remark}
    The Fourier slice theorem can be extended to apply to members of $\Sch'(\R^d)$ so long as some additional care is taken regarding in what sense the equality in \cref{eq:Fourier-slice-S} holds~\cite[Theorem~7.7]{Parhikplane}.
\end{remark}

Let $\Sch_k \coloneqq 
\RadonOp_k\paren*{\Sch(\R^d)}$ denote the range of the $k$-plane transform. The range $\Sch_k$ is a strict subspace of $\Sch(\Xi_k)$ that satisfies certain consistency conditions (see~\cite[Chapter~4]{MarkoeAnalyticTomo} for a detailed discussion and references on this matter). We have the following additional result regarding the continuity and invertibility of the $k$-plane transform.

\begin{proposition}[{\cite[Corollary~5.3]{Parhikplane}}] \label[proposition]{prop:homeo}
    The operator $\RadonOp_k: \Sch(\R^d) \to \Sch_k$ is a homeomorphism with inverse $\RadonOp_k^{-1} = \RadonOp_k^*\KOp_k: \Sch_k \to \Sch(\R^d)$.
\end{proposition}

\Cref{prop:homeo} motivates the following \emph{distributional extension} of the $k$-plane transform and related operators.

\begin{definition}[{\cite[Definition~6.1]{Parhikplane}}] \label[definition]{defn:distributional-k-plane} \hfill

\begin{enumerate}
    \item The \emph{distributional $k$-plane transform}
    \begin{equation}
        \RadonOp_k: \Sch'(\R^d) \to \paren[\big]{\KOp_{d-k} \RadonOp_k\paren[\big]{\Sch(\R^d)}}'
    \end{equation}
    is defined to be the dual map of the homeomorphism $\RadonOp_k^*: \KOp_{d-k} \RadonOp_k\paren*{\Sch(\R^d)} \to \Sch(\R^d)$.
    \label{item:distributional-k-plane}

    \item The \emph{distributional filtered $k$-plane transform}
    \begin{equation}
        \KOp_{d-k} \RadonOp_k: \Sch'(\R^d) \to \Sch_k'
    \end{equation}
    is defined to be the dual map of the homeomorphism $\RadonOp_k^* \KOp_{d-k}: \Sch_k \to \Sch(\R^d)$.
    \label{item:distributional-filtered-k-plane}

    \item The \emph{distributional backprojection}
    \begin{equation}
        \RadonOp_k^*: \Sch_k' \to \Sch'(\R^d)
    \end{equation}
    is defined to be the dual map of the homeomorphism $\RadonOp_k: \Sch(\R^d) \to \Sch_k$.

    \item The \emph{extended distributional backprojection}
    \begin{equation}
        \RadonOp_k^*: \Sch'(\Xi_k) \to \Sch'(\R^d)
    \end{equation}
    is defined to be the dual map of the continuous operator $\RadonOp_k: \Sch(\R^d) \to \Sch(\Xi_k)$, which is well-defined since $\Sch_k$ is continuously embedded in $\Sch(\Xi_k)$.
\end{enumerate}
\end{definition}

Based on these definitions, we state in \cref{thm:cont-inv-dual} a result on the invertibility of the filtered $k$-plane transform on $\Sch'(\R^d)$, which is the dual of \cref{prop:cont-inv,prop:homeo}.
\begin{theorem} \label{thm:cont-inv-dual}
    It holds that $\RadonOp_k^* \KOp_{d-k} \RadonOp_k = \Id$ on $\Sch'(\R^d)$. Moreover, the filtered $k$-plane transform $\KOp_{d-k} \RadonOp_k: \Sch'(\R^d) \to \Sch_k'$ is a homeomorphism with inverse given by the backprojection $(\KOp_{d-k} \RadonOp_k)^{-1} = \RadonOp_k^*: \Sch_k' \to \Sch'(\R^d)$.
\end{theorem}

While this setup provides an attractive formulation to handle the distributional extension of the $k$-plane transform and its dual, it turns out that the distribution spaces $\Sch_k'$ and $\paren*{\KOp_{d-k} \RadonOp_k\paren*{\Sch(\R^d)}}'$ are actually \emph{equivalence classes} of distributions~\cite{Parhikplane}. Luckily, by working with certain Banach subspaces that continuously embed into these distribution spaces, one can identify a concrete member of the equivalence classes via continuous projection operators. To this end, the results of this paper hinge on a nontrivial result of~\cite{Parhikplane} regarding the invertibility of the distributional dual $k$-plane transform on the space of isotropic Radon measures.

Consider the operator
\begin{equation}
    \P_\iso\curly{g}(\mat{A}, \vec{t}) = \avint_{\mathrm{O}_{d-k}(\R)} g(\mat{U}\mat{A}, \mat{U}\vec{t}) \dd\sigma(\mat{U}),
\end{equation}
where $\avint$ denotes the average integral and $\sigma$ is the Haar measure on $\rmO_{d-k}(\R)$. This is a well-defined operator that maps $C_0(\Xi_k) \to C_0(\Xi_k)$ and, in particular, is the self-adjoint continuous projector which extracts the isotropic part of a function~\cite[Equation~(8.10)]{Parhikplane}. By the Riesz--Markov--Kakutani representation theorem, we can extend $\P_\iso = \P_\iso^*$ by duality to act on $\M(\Xi_k) = \paren*{C_0(\Xi_k)}'$, the Banach space of finite Radon measures on $\Xi_k$. To this end, define
\begin{equation}
    \M_\iso(\Xi_k) \coloneqq \P_\iso\paren*{\M(\Xi_k)},
\end{equation}
the Banach subspace of isotropic finite Radon measures on $\Xi_k$. This Banach subspace is complemented in $\M(\Xi_k)$~\cite[Theorem~8.2]{Parhikplane} and so
\begin{equation}
   \M(\Xi_k) = \M_\iso(\Xi_k) \oplus \paren{\M_\iso(\Xi_k)}^\comp
\end{equation}
with $\paren{\M_\iso(\Xi_k)}^\comp \coloneqq (\Id - \P_\iso)\paren*{\M(\Xi_k)}$. We also note that
\begin{equation}
    \M_\iso(\Xi_k) = \paren*{C_{0, \iso}(\Xi_k)}',
    \label{eq:Miso-Ciso-dual}
\end{equation}
where
\begin{equation}
    C_{0, \iso}(\Xi_k) \coloneqq \P_\iso\paren*{C_0(\Xi_k)} = \cl{(\Sch_k, \norm{\dummy}_{L^\infty})},
    \label{eq:C0-dense}
\end{equation}
where the last equality is from~\cite[Equation~(8.12)]{Parhikplane}.
\begin{proposition}[{see~\cite[Theorems~8.1~and~8.2~and~Corollary~8.3]{Parhikplane}}] \label[proposition]{prop:Miso-k-plane}
    The Banach space $\M_\iso(\Xi_k)$ continuously embeds into $\Sch_k'$. Further, the distributional backprojection operator $\RadonOp_k^*$ is invertible on $\M_\iso(\Xi_k)$, so that
    \begin{equation}
        \KOp_{d-k}\RadonOp_k \RadonOp_k^* = \Id \text{ on } \M_\iso(\Xi_k).
    \end{equation}
    Furthermore, the null space of $\RadonOp_k^*: \M(\Xi_k) \to \Sch'(\R^d)$ is $\paren{\M_\iso(\Xi_k)}^\comp$ and so $\RadonOp_k^*\paren*{\M(\Xi_k)} = \RadonOp_k^*\paren*{\M_\iso(\Xi_k)}$.
\end{proposition}

\subsubsection{The Case \texorpdfstring{$k=0$}{k=0}} \label{subsubsec:k=0}
When $k = 0$, the $k$-plane transform of $\varphi \in \Sch(\R^d)$ becomes
\begin{equation}
    \RadonOp_0\curly{\varphi}(\mat{A}, \vec{t}) = \int_{\R^d} \varphi(\vec{x}) \delta(\mat{A}\vec{x} - \vec{t}) \dd\vec{x} = \varphi(\mat{A}^\T\vec{t}).
\end{equation}
Indeed, when $k = 0$, the matrix $\mat{A} \in \R^{d\times d}$ is now an orthogonal matrix (i.e., $\mat{A}^\T\mat{A} = \mat{A}\mat{A}^\T = \mat{I}_d$) and $\vec{t} \in \R^d$. The last equality of the above integral follows from the change of variables $\vec{y} = (\mat{A} \vec{x} - \vec{t})$. In this case, it is clear that the range $\Sch_0$ is exactly the closed subspace of isotropic functions in $\Sch(\Xi_0)$, denoted by $\Sch_\iso(\Xi_0)$. The fundamental results discussed previously trivially hold in this limit setting. In particular, it is easy to check that
\begin{equation}
    c_{d, 0} \RadonOp_0^*\RadonOp_0 = \Id \text{ on } \Sch(\R^d),
\end{equation}
which is the same statement as $\RadonOp_0^*\KOp_d\RadonOp_0 = \Id$ since $\KOp_d = c_{d, 0} \Id$ (\cref{prop:cont-inv}). Therefore, in the remainder of the paper, when working with the $k$-plane transform for general $k$ such that $0 \leq k < d$, we shall not treat separately the case $k = 0$. We do warn the reader, however, to be aware that the underlying mathematics of the $(k=0)$-plane transform is much simpler than that of $1 \leq k < d$.

\subsection{Polynomial Spaces and Related Projectors} \label{subsec:poly-proj}
The null space of the regularizer in the learning problem \cref{eq:variational-problem-intro} is the space of polynomials of degree $n_{\LOp}$ (see \cref{lemma:null-space}), which depends on the operator $\LOp$. This space is denoted $\Poly_{n_{\LOp}}(\R^d)$. To this end, we will be interested in working with a biorthogonal system for this null space. We will use the biorthogonal system from~\cite[Section~2.2]{UnserRidges}.

\begin{remark}
    When the null space of $\LOp$ is trivial, we make the identifications $n_{\LOp} = (-1)$ and $\Poly_{n_{\LOp}}(\R^d) = \curly{0}$, noting that any sum taken from $n = 0$ to $(-1)$ is understood as $0$.
\end{remark}

The space $\Poly_{n_{\LOp}}(\R^d)$ is spanned by the monomial/Taylor basis
\begin{equation}
    m_{\vec{n}}(\vec{x}) = \frac{\vec{x}^\vec{n}}{\vec{n}!},
    \label{eq:poly}
\end{equation}
where $\vec{n} = (n_1, \ldots, n_d) \in \N_0^d$ is a multi-index. Accordingly, we have that
\begin{equation}
    \Poly_{n_{\LOp}}(\R^d) = \curly*{\sum_{\abs{\vec{n}} \leq n_{\LOp}} b_\vec{n} m_\vec{n} \st b_\vec{n} \in \R} \subset \Sch'(\R^d).
\end{equation}

Importantly, $\Poly_{n_{\LOp}}(\R^d)$ forms a finite-dimensional Banach subspace of $\Sch'(\R^d)$. Since all norms are equivalent in finite dimensions, the exact choice does not matter, but, for concreteness, we equip $\Poly_{n_{\LOp}}(\R^d)$ with the norm
\begin{equation}
    \norm{p}_{\Poly_{n_{\LOp}}} \coloneqq \norm{(b_\vec{n})_{\abs{\vec{n}} \leq n_{\LOp}}}_2,
\end{equation}
where the $b_\vec{n}$ are the coefficients of $p$ in the monomial/Taylor basis. While the dual space $\Poly_{n_{\LOp}}'(\R^d)$ is also finite-dimensional, its ``abstract'' elements are actually equivalence classes in $\Sch'(\R^d)$. Following the approach from~\cite[Section~2.2]{UnserRidges}, we identify every dual element $p^* \in \Poly_{n_{\LOp}}'(\R^d)$ as a function in $\Sch(\R^d)$ by working with a concrete dual basis $\curly{m_\vec{n}^*}_{\abs{\vec{n}} \leq n_{\LOp}}$ that satisfies the biorthogonality property $\ang{m_\vec{n}^*, m_{\vec{n}'}} = \delta[\vec{n} - \vec{n}']$, where $\delta[\dummy]$ is the Kronecker impulse which takes the value $1$ when its input is $0$ and $0$ otherwise. Our specific choice is
\begin{equation}
    m_\vec{n}^* \coloneqq (-1)^{\abs{\vec{k}}} \partial^\vec{n} \kappa_\iso \in \Sch(\R^d),
\end{equation}
where $\kappa_\iso \in \Sch(\R^d)$ is an isotropic function constructed in~\cite[Lemma~1]{UnserRidges}. Its frequency response is such that $\hat{\kappa}_\iso(\vec{\xi}) = \hat{\kappa}_\rad(\norm{\vec{\xi}}_2)$, where the radial frequency profile $\hat{\kappa}_\rad \in \Sch(\R)$ satisfies $0 \leq \hat{\kappa}_\rad \leq 1$ and, for $\abs{\omega} \geq 1$, $\hat{\kappa}_\rad(\omega) = 0$.

The biorthogonal system $\curly{(m_\vec{n}^*, m_\vec{n})}_{\abs{\vec{n}} \leq n_{\LOp}}$ allows us to define the projection onto $\Poly_{n_{\LOp}}(\R^d)$ by the operator
\begin{equation}
    \P_{\Poly_{n_{\LOp}}(\R^d)}\curly{f} = \sum_{\vec{n} \leq n_{\LOp}} \ang{m_\vec{n}^*, f} m_\vec{n}.
\end{equation}
This projector continuously maps $\Sch'(\R^d) \to \Poly_{n_{\LOp}}(\R^d)$ (because $m_\vec{n}^* \in \Sch(\R^d)$). Moreover, since $L^\infty_{-n_{\LOp}}(\R^d)$ continuously embeds into $\Sch'(\R^d)$, the restricted operator
\begin{equation}
    \P_{\Poly_{n_{\LOp}}(\R^d)}:L^\infty_{-n_{\LOp}}(\R^d) \to \Poly_{n_{\LOp}}(\R^d)
\end{equation}
is continuous as well. Lastly, the finite dimensionality of $\Poly_{n_{\LOp}}(\R^d)$ ensures that $\Poly_{n_{\LOp}}(\R^d)$ is complemented in $L^\infty_{-n_{\LOp}}(\R^d)$~\cite[Lemma~4.21]{RudinFA}. Thus, the complementary projector
\begin{equation}
    (\Id - \P_{\Poly_{n_{\LOp}}(\R^d)}):L^\infty_{-n_{\LOp}}(\R^d) \to L^\infty_{-n_{\LOp}}(\R^d)
    \label{eq:comp-proj-poly}
\end{equation}
is guaranteed to exist and be continuous.

\section{Main Results}
We first define the class of operators that are admissible for the learning problem in \cref{eq:variational-problem-intro}. These operators form a subclass of the so-called \emph{spline-admissible} operators~\cite[Definition~1]{UnsergTV}.
\begin{definition} \label[definition]{defn:admissible}
    A continuous LSI operator $\LOp: \Sch(\R^d) \to \Sch'(\R^d)$ is said to be \emph{$k$-plane-admissible} with a polynomial null space of degree $n_{\LOp}$ if
    \begin{enumerate}
        \item its adjoint $\LOp^*$ is a continuous injection that maps $\Sch(\R^d)$ into $L^1_{n_{\LOp}}(\R^d)$; \label{item:L*-mapping}
        \item it is isotropic in the sense that its frequency response is continuous and satisfies $\hat{L}(\vec{\xi}) = \hat{L}_\rad(\norm{\vec{\xi}}_2)$ for some continuous univariate radial frequency profile $\hat{L}_\rad: \R \to \R$; \label{item:isotropic}
        \item the radial frequency profile $\hat{L}_\rad$ does not vanish over $\R$, except for a zero of order $\gamma_{\LOp} \in (n_{\LOp}, n_{\LOp} + 1]$ at the origin, so that there exists a constant $C > 0$ satisfying
        \begin{equation}
            \lim_{\omega \to 0} \frac{\hat{L}_\rad(\omega)}{\abs{\omega}^{\gamma_{\LOp}}} = C;
        \end{equation}
        \label{item:zero-cancel}
        \item there exist $\gamma_{\LOp}' > (d - k)$, $C' > 0$, and $R > 0$ such that
        \begin{equation}
            \abs{\hat{L}_\rad(\omega)} \geq C' \abs{\omega}^{\gamma_{\LOp}'}
        \end{equation}
        for all $\abs{\omega} > R$.
        \label{item:order-high}
    \end{enumerate}
\end{definition}
\begin{remark}  \label[remark]{rem:L-ann}
    \Cref{item:zero-cancel} implies that the extension by duality $\LOp: L^\infty_{-n_{\LOp}}(\R^d) \to \Sch'(\R^d)$ of a $k$-plane-admissible operator annihilates polynomials of degree at most $n_{\LOp}$.
\end{remark}
\begin{remark}
\Cref{item:order-high} guarantees that the order of the operator is sufficiently high to ensure that the point evaluation functional is well-defined on the native space $\M_{\LOp}^k(\R^d)$. In particular, this condition is reminiscent of the condition $s > (d - k)$ on the smoothness index of an $L^1$-Sobolev space defined on (subsets of) $\R^{d-k}$ to ensure continuity of its members. Indeed, when $s > (d - k)$, the Sobolev embedding theorem guarantees that the $L^1$-Sobolev space of order $s$ embeds into the space of continuous functions. In our setting, the condition $\gamma_{\LOp}' > (d - k)$ plays the role of the smoothness index.
This property is key to establishing the existence of solutions to the learning problem \cref{eq:variational-problem-intro}.
\end{remark}

\subsection{Native Spaces} \label{sec:native-spaces}
The primary technical contribution of this paper is the careful treatment of the distributional extension and (pseudo)invertibility of the operator $\LOp_{\RadonOp_k} \coloneqq \KOp_{d-k}\RadonOp_k\LOp$, where $\LOp$ is a $k$-plane admissible operator (\cref{defn:admissible}). This in turn allows us to define our native (Banach) spaces. Indeed, for the definition of the native space in \cref{eq:native-space-intro} to be coherent, the action of $\LOp_{\RadonOp_k}$ on $L^\infty_{-n_{\LOp}}(\R^d)$ must be well-defined. Furthermore, as we shall see in \cref{lemma:null-space}, the null space of $\LOp_{\RadonOp_k}$ is exactly the space of polynomials of degree at most $n_{\LOp}$. We then use a technique from spline theory to ``factor out'' the null space of $\LOp_{\RadonOp_k}$ and identify the subspace of $\M_{\LOp}^k(\R^d)$ on which $\LOp_{\RadonOp_k}$ is invertible~\cite{deBoorSmoothingSplines,duchon1977splines}. This allows us identify $\M_{\LOp}^k(\R^d)$ as the direct sum of two Banach spaces. Thus, it forms a Banach space when equipped with the composite norm.

\begin{lemma} \label[lemma]{lemma:well-defined}
    Let $\LOp$ be a $k$-plane-admissible operator in the sense of \cref{defn:admissible}. Then, the operator $\LOp_{\RadonOp_k} = \KOp_{d-k}\RadonOp_k\LOp$ continuously maps $L^\infty_{-n_{\LOp}}(\R^d) \to \Sch_k'$.
\end{lemma}
\begin{proof}
    It suffices to prove that the adjoint operator $\LOp_{\RadonOp_k}^* = \LOp^* \RadonOp_k^* \KOp_{d-k}$ continuously maps $\Sch_k \to L^1_{n_{\LOp}}(\R^d)$. The result then follows by duality. Since $\RadonOp_k^* \KOp_{d-k}: \Sch_k \to \Sch(\R^d)$ is a homeomorphism (\cref{prop:homeo}), the lemma follows from \cref{item:L*-mapping} in \cref{defn:admissible}.
\end{proof}

\begin{lemma} \label[lemma]{lemma:null-space}
    Let $\LOp$ be a $k$-plane-admissible operator in the sense of \cref{defn:admissible}. Then, the null space of the operator $\LOp_{\RadonOp_k} = \KOp_{d-k}\RadonOp_k\LOp: L^\infty_{-n_{\LOp}}(\R^d) \to \Sch_k'$ is the space of polynomials of degree at most $n_{\LOp}$ on $\R^d$, denoted by $\Poly_{n_{\LOp}}(\R^d)$.
\end{lemma}
\begin{proof}
    Let $\Null(\LOp_{\RadonOp_k})$ denote the null space of $\LOp_{\RadonOp_k}$. More precisely,
    \begin{equation}
        \Null(\LOp_{\RadonOp_k}) = \curly{f \in L^\infty_{-n_{\LOp}}(\R^d) \st \LOp_{\RadonOp_k}\curly{f} = 0}.
    \end{equation}
    From \cref{rem:L-ann}, $\LOp$ annihilates polynomials of degree at most $n_{\LOp}$. This implies that $\Null(\LOp_{\RadonOp_k}) \supset \Poly_{n_{\LOp}}(\R^d)$. However, given $f \in L^\infty_{-n_{\LOp}}(\R^d) \subset \Sch'(\R^d)$, from the Fourier slice theorem,
    \begin{equation}
        \reallywidehat{\LOp_{\RadonOp_k}\curly{f}(\mat{A}, \dummy)}(\vec{\omega}) = c_{d,k} \norm{\vec{\omega}}_2^k \hat{L}(\mat{A}^\T\vec{\omega}) \hat{f}(\mat{A}^\T\vec{\omega}) = c_{d,k} \norm{\vec{\omega}}_2^k \hat{L}_\rad(\norm{\vec{\omega}}_2) \hat{f}(\mat{A}^\T\vec{\omega}).
        \label{eq:null-space-proof}
    \end{equation}
    Note that the product in the last equality must be a well-defined tempered distribution (in the $\vec{\omega}$ variable) since $\LOp_{\RadonOp_k} :L^\infty_{-n_{\LOp}}(\R^d) \to \Sch_k'$ is well-defined by \cref{lemma:well-defined}. Due to the vanishing property in \cref{item:zero-cancel} of \cref{defn:admissible}, this quantity is $0$ if and only if $\hat{f}$ is supported only at $\vec{0}$, in which case $f$ is a polynomial. Combined with the growth restriction $f \in L^\infty_{-n_{\LOp}}(\R^d)$, we have that $f$ must be a polynomial of degree at most $n_{\LOp}$. Therefore, $\Null(\LOp_{\RadonOp_k}) \subset \Poly_{n_{\LOp}}(\R^d)$.
\end{proof}

From \cref{lemma:well-defined,lemma:null-space}, we deduce that the native space in \cref{eq:native-space-intro} is well-defined and that the null space of $\LOp_{\RadonOp_k}$ is the finite-dimensional Banach space $\Poly_{n_{\LOp}}(\R^d)$. The next two technical theorems (\cref{thm:right-inverse,thm:Banach-structure}) establish the Banach structure of the native space.

\begin{theorem} \label{thm:right-inverse}
    Let $\LOp$ be a $k$-plane-admissible operator in the sense of \cref{defn:admissible}. Then, the operator $\LOp_{\RadonOp_k} = \KOp_{d-k}\RadonOp_k\LOp$ maps $\M_{\LOp}^k(\R^d) \to \M_\iso(\Xi_k)$. Furthermore, there exists an operator $\LOp_{\RadonOp_k}^\dagger$ that continuously maps $\M_\iso(\Xi_k) \to L^\infty_{-n_{\LOp}}(\R^d) \subset \Sch'(\R^d)$ and is such that
    \begin{align}
        \LOp_{\RadonOp_k} \LOp_{\RadonOp_k}^\dagger \curly{\mu} = \mu & \text{ for all } \mu \in \M_\iso(\Xi_k), \label{eq:isometry-Miso} \\
        \LOp_{\RadonOp_k}^\dagger \LOp_{\RadonOp_k} \curly{f} = (\Id - \P_{\Poly_{n_{\LOp}}(\R^d)})\curly{f} & \text{ for all } f \in \M_{\LOp}^k(\R^d). \label{eq:comp-proj-inverse}
    \end{align}
    This operator is realized by
    \begin{equation}
        \LOp_{\RadonOp_k}^{\dagger} = (\Id - \P_{\Poly_{n_{\LOp}}(\R^d)})\LOp^{-1} \RadonOp_k^*,
        \label{eq:LOp-k-defn}
    \end{equation}
    where $\LOp^{-1}$ is the operator specified by the frequency response $\vec{\xi} \mapsto 1/\hat{L}(\vec{\xi})$. Furthermore, $\LOp_{\RadonOp_k}^{\dagger}$ is an integral operator specified by the kernel $(\vec{x}, (\mat{A}, \vec{t})) \mapsto g_{\mat{A}, \vec{t}}(\vec{x})$ that takes the form
    \begin{equation}
        g_{\mat{A}, \vec{t}}(\vec{x}) = \rho_{\LOp}(\mat{A}\vec{x} - \vec{t}) - \sum_{\abs{\vec{k}} \leq n_{\LOp}} \ang{m_\vec{k}^* , \rho_{\LOp}(\mat{A}(\dummy) - \vec{t})} m_\vec{k}(\vec{x}),
        \label{eq:g-kernel}
    \end{equation}
    where $\rho_{\LOp} = \FourierOp^{-1}_{d-k}\curly{1 / \hat{L}_\rad(\norm{\dummy}_2)}$ and $\curly{(m_\vec{n}^*, m_\vec{n})}_{\abs{\vec{n}} \leq n_{\LOp}}$ are specified in \cref{subsec:poly-proj}, and where $\FourierOp^{-1}_{d-k}$ denotes the $(d-k)$-variate inverse Fourier transform. This kernel satisfies the stability/continuity bound
    \begin{equation}
        \sup_{\substack{\vec{x} \in \R^d \\ (\mat{A}, \vec{t}) \in \Xi_k}} \abs{g_{\mat{A}, \vec{t}}(\vec{x})} (1 + \norm{\vec{x}}_2)^{-n_{\LOp}} < \infty
        \label{eq:stability-bound}
    \end{equation}
    with
    \begin{equation}
        (\mat{A}, \vec{t}) \mapsto g_{\mat{A}, \vec{t}}(\vec{x}_0) \in C_{0, \iso}(\Xi_k),
        \label{eq:kernel-C0}
    \end{equation}
    for any fixed $\vec{x}_0 \in \R^d$. Thus, for $\mu \in \M_\iso(\Xi_k)$,
    \begin{equation}
        \LOp_{\RadonOp_k}^\dagger\curly{\mu}(\vec{x}) = \int_{\Xi_k} g_{\mat{A}, \vec{t}}(\vec{x}) \dd \mu(\mat{A}, \vec{t}), \quad \vec{x} \in \R^d.
        \label{eq:Linv-integral-representation}
    \end{equation}
\end{theorem}

\begin{theorem} \label{thm:Banach-structure}
    Consider the setting of \cref{thm:right-inverse}. Then, the following hold.
    \begin{enumerate}
        \item The range space $\mathcal{V} \coloneqq \LOp_{\RadonOp_k}^\dagger\paren*{\M_\iso(\Xi_k)}$ is a Banach space when equipped with the norm
        \begin{equation}
            \norm{f}_{\mathcal{V}} \coloneqq \norm{\LOp_{\RadonOp_k} f}_\M.
            \label{eq:U-norm}
        \end{equation}
        This Banach space is isometrically isomorphic to $\M_\iso(\Xi_k)$.
        \label{item:iso-Miso}

        \item
        \label{item:direct-sum}
        The native space $\M_{\LOp}^k(\R^d)$ is decomposable as the direct sum of Banach spaces
        \begin{equation}
            \M_{\LOp}^k(\R^d) = \mathcal{V} \oplus \Poly_{n_{\LOp}}(\R^d) = \LOp_{\RadonOp_k}^\dagger\paren*{\M_\iso(\Xi_k)} \oplus \Poly_{n_{\LOp}}(\R^d).
        \end{equation}
        
        \item The native space $\M_{\LOp}^k(\R^d)$ forms a \textit{bona fide} Banach space when equipped with the norm
        \begin{equation}
            \norm{f}_{\M_{\LOp}^k} \coloneqq \norm{\LOp_{\RadonOp_k} f}_\M + \norm{\P_{\Poly_{n_{\LOp}}(\R^d)} f}_{\Poly_{n_{\LOp}}}.
        \end{equation}
        Furthermore, $\M_{\LOp}^k(\R^d)$ is isometrically isomorphic to $\M_\iso(\Xi_k) \times \Poly_{n_{\LOp}}(\R^d)$ via the map
        \begin{equation}
            f = \LOp_{\RadonOp_k}^\dagger\curly{u} + p \mapsto (u, p),
        \end{equation}
        where $u = \LOp_{\RadonOp_k}\curly{f}$ and $p = \P_{\Poly_{n_{\LOp}}(\R^d)} f$.
        \label{item:native-space-bona-fide}

        \item For any $\vec{x}_0 \in \R^d$, the shifted Dirac impulse (point evaluation functional) $\delta(\dummy - \vec{x}_0): f \mapsto f(\vec{x}_0)$ is weak$^*$-continuous\footnote{In particular, we prove that $(\M_{\LOp}^k(\R^d), \norm{\dummy}_{\M_{\LOp}^k})$ can be identified as the dual of some primary Banach space, which allows us to equip $\M_{\LOp}^k(\R^d)$ with a weak$^*$ topology.} on $\M_{\LOp}^k(\R^d)$. \label{item:weak*}
    \end{enumerate}
\end{theorem}
The proofs of \cref{thm:right-inverse,thm:Banach-structure} appear in \cref{app:right-inverse,app:Banach-structure}, respectively.

\subsection{Optimality of Neural Architectures With Multivariate Nonlinearities}

Having established the properties of the native space $\M_{\LOp}^k(\R^d)$ in \cref{sec:native-spaces}, we can now prove our main theorem regarding the function-space optimality of neural architectures with multivariate nonlinearities.

\begin{theorem} \label{thm:representer}
  Let $\Loss(\dummy, \dummy): \R \times \R \to \R$ be convex, coercive, and lower-semicontinuous in
  its second argument and let $\LOp$ be a $k$-plane-admissible operator in the sense of
  \cref{defn:admissible}. Then, for any finite data set $\curly{(\vec{x}_m,
  y_m)}_{m=1}^M \subset \R^d \times \R$ for which the data-fitting problem is well-posed\footnote{The data-fitting problem is well-posed over the null space when the classical least-squares polynomial-fitting problem admits a unique solution with respect to the data $\curly{(\vec{x}_m,
  y_m)}_{m=1}^M \subset \R^d \times \R$.} over $\mathcal{P}_{n_{\LOp}}(\R^d)$, the solution
  set to the data-fitting variational problem
  \begin{equation}
    S \coloneqq \argmin_{f \in \M_{\LOp}^k(\R^d)} \: \sum_{m=1}^M \Loss(y_m, f(\vec{x}_m)) + \lambda \norm{\KOp_{d-k}\RadonOp_k\LOp f}_\M
    \label{eq:variational-problem}
  \end{equation}
  is nonempty, convex, and weak$^*$-compact. If $\Loss(\dummy, \dummy)$ is
  strictly convex (or if it imposes the equality $y_m = f(\vec{x}_m)$ for $m =
  1, \ldots, M$), then $S$ is the weak$^*$ closure of the convex
  hull of its extreme points, which can all be expressed as
  \begin{equation}
    f_\mathrm{extreme}(\vec{x}) = c(\vec{x}) + \sum_{n=1}^N v_n \, \rho_{\LOp}(\mat{A}_n\vec{x} -
      \vec{t}_n),
    \label{eq:extreme-points-rep-thm}
  \end{equation}
  where the number $N$ of atoms satisfies $N \leq (M - \dim \Poly_{n_{\LOp}}(\R^d))$, and, for $n = 1, \ldots, N$, $v_n \in \R \setminus \curly{0}$, $\mat{A}_n \in V_{d-k}(\R^d)$, and $\vec{t}_n \in \R^{d-k}$. The function $c \in \Poly_{n_{\LOp}}(\R^d)$ is a polynomial of degree at most $n_{\LOp}$ and $\rho_{\LOp}: \R^{d-k} \to \R^d$ is a $(d-k)$-variate nonlinearity given by $\rho_{\LOp} = \FourierOp_{d-k}^{-1} \curly{1 / \hat{L}_\rad(\norm{\dummy}_2)}$, where $\FourierOp_{d-k}^{-1}$ is the $(d-k)$-variate inverse Fourier transform. Finally, the regularization cost, which is common to all solutions, is
  \begin{equation}
    \norm{\KOp_{d-k}\RadonOp_k\LOp f_\mathrm{extreme}}_\M = \sum_{n=1}^N \abs{v_n} = \norm{\vec{v}}_1.
  \end{equation}
\end{theorem}
The proof of the theorem requires the following proposition.
\begin{proposition}[{\cite[Lemma~10.2]{Parhikplane}}] \label[proposition]{prop:Dirac-iso}
The isotropic shifted Dirac impulse $\delta_\iso(\dummy - (\mat{A}_0, \vec{t}_0)) \in \M_\iso(\Xi_k)$ defined by
\begin{equation}
    \delta_\iso(\dummy - (\mat{A}_0, \vec{t}_0)) \coloneqq \P_\iso\curly{\delta(\dummy - (\mat{A}_0, \vec{t}_0)))},
\end{equation}
where $\delta(\dummy - (\mat{A}_0, \vec{t}_0)) = \delta(\dummy - \mat{A}_0) \delta(\dummy - \vec{t}_0) \in \M(\Xi_k)$ is the ``classical'' Dirac impulse on $\Xi_k$, satisfies the following properties.
\begin{enumerate}
    \item
    \label{item:k-plane-Dirac-sampling}
    Sampling: For any $\psi \in C_{0, \iso}(\Xi_k)$,
    \begin{equation}
        \ang{\delta_\iso(\dummy - (\mat{A}_0, \vec{t}_0)), \psi}_k = \psi(\mat{A}_0, \vec{t}_0).
    \end{equation}
    \item
    \label{item:iso-Dirac-iso}
    Rotation invariance: For any $\mat{U} \in \rmO_{d-k}(\R)$,
    \begin{equation}
        \delta_\iso(\dummy - (\mat{A}_0, \vec{t}_0)) = \delta_\iso(\dummy - (\mat{U}\mat{A}_0, \mat{U}\vec{t}_0)).
    \end{equation}
    \item
    \label{item:iso-Dirac-unit-norm}
    Unit norm: $\norm{\delta_\iso(\dummy - (\mat{A}_0, \vec{t}_0))}_\M = 1$. 

    \item 
    \label{item:distinct-ell1}
    Linear combination: For any set $\curly{(\mat{A}_n, \vec{t}_n)}_{n=1}^N \subset \Xi_k$ of distinct points,
    \begin{equation}
        \norm*{\sum_{n=1}^N a_n \delta_\iso(\dummy - (\mat{A}_n, \vec{t}_n))}_\M = \sum_{n=1}^N \abs{a_n} = \norm{\vec{a}}_1.
        \label{eq:distinct-ell1}
    \end{equation}
    
    \item Extreme points of  $B_{\M_\iso} \coloneqq \curly{e \in \M_\iso(\Xi_k) \st \norm{e}_{\M} \leq 1}$: If $e \in \Ext B_{\M_\iso}$, then $e = \pm \delta_\iso(\dummy - (\mat{A}_n, \vec{t}_n))$ for some $(\mat{A}_n, \vec{t}_n) \in \Xi_k$.

    \label{item:Miso-extreme}
\end{enumerate}
\end{proposition}

\begin{proof}[Proof of \Cref{thm:representer}]
    The proof relies on the abstract representer theorem in~\cite{UnserDirectSums} (see also~\cite{BoyerRepresenter,BrediesSparsity,UnserUnifyingRepresenter}). From the assumptions on the loss function combined with the weak$^*$-continuity of the point evaluation functional on $\M_{\LOp}^k(\R^d)$ (\cref{item:weak*} in \cref{thm:Banach-structure}), our setting coincides with the hypotheses of~\cite[Theorem~3]{UnserDirectSums}. First, this abstract result ensures that the solution set $S$ is nonempty, convex, and weak$^*$-compact. Second, it ensures that, when the loss function is strictly convex (or if it imposes the equality $y_m = f(\vec{x}_m)$ for $m = 1, \ldots, M$), $S$ is the weak$^*$-closure of the convex hull of its extreme points, which can all be expressed as
    \begin{equation}
        f_\textrm{extreme}(\vec{x}) = c(\vec{x}) + \sum_{n=1}^{N} v_n e_n(\vec{x}),
    \end{equation}
    where the number $N$ of atoms satisfies $N \leq (M - \dim \Poly_{n_{\LOp}}(\R^d))$, $c(\dummy)$ is in the null space of the regularizer (i.e., $c \in \Poly_{n_{\LOp}}(\R^d)$), and, for $n = 1, \ldots, N$, $v_n \in \R \setminus \curly{0}$ and $e_n$ is an extreme point of the unit regularization ball
    \begin{equation}
        B \coloneqq \curly*{f \in \M_{\LOp}^k(\R^d) \st \norm{\KOp_{d-k}\RadonOp_k\LOp f}_\M \leq 1}.
    \end{equation}
    
    From \cref{item:direct-sum} in \cref{thm:Banach-structure}, we have the direct-sum decomposition of the native space as
    \begin{equation}
        \M_{\LOp}^k(\R^d) = \mathcal{V} \oplus \Poly_{n_{\LOp}}(\R^d) = \LOp_{\RadonOp_k}^\dagger\paren*{\M_\iso(\Xi_k)} \oplus \Poly_{n_{\LOp}}(\R^d).
    \end{equation}
    Since $\LOp_{\RadonOp_k}^\dagger: \M_\iso(\Xi_k) \to \V$ is an isometric isomorphism (\cref{item:iso-Miso} in \cref{thm:Banach-structure}), the extreme points of $B$ take the form $\LOp_{\RadonOp_k}^\dagger\paren*{\Ext B_{\M_\iso}}$ plus a polynomial term in $\Poly_{n_{\LOp}}(\R^d)$. From \cref{item:Miso-extreme} in \cref{prop:Dirac-iso}, it then follows that
    \begin{equation}
        e_n = \LOp_{\RadonOp_k}^\dagger\curly{\pm \delta_\iso(\dummy - (\mat{A}_n, \vec{t}_n))} + p = \pm \rho_{\LOp}(\mat{A}_n(\dummy) - \vec{t}_n) + \tilde{p},
    \end{equation}
    with $(\mat{A}_n, \vec{t}_n) \in \Xi_k$ and $p, \tilde{p} \in \Poly_{n_{\LOp}}(\R^d)$.

    A calculation in the Fourier domain reveals that $\LOp\curly{\rho_{\LOp}(\mat{A}(\dummy) - \vec{t})} = \delta(\mat{A}(\dummy) - \vec{t})$ for $(\mat{A}, \vec{t}) \in \Xi_k$. Next, we invoke the property that $\KOp_{d-k} \RadonOp_k\curly{\delta(\mat{A}(\dummy) - \vec{t})} = \delta_\iso(\dummy - (\mat{A}, \vec{t}))$~\cite[Equation~(9.16)]{Parhikplane}, which yields that
    \begin{equation}
        \KOp_{d-k}\RadonOp_k\LOp\curly{f_\mathrm{extreme}} = \sum_{n=1}^N v_n \delta_\iso(\dummy - (\mat{A}_n, \vec{t}_n)).
    \end{equation}
    Finally, by \cref{item:distinct-ell1} in \cref{prop:Dirac-iso}, we have that $\norm{\KOp_{d-k}\RadonOp_k\LOp f_\mathrm{extreme}}_\M = \sum_{n=1}^N \abs{v_n} = \norm{\vec{v}}_1$, which proves the theorem.
\end{proof}

\subsubsection{Discussion} \label{sec:discussion}
The main takeaway from \cref{thm:representer} is that \emph{sparse} neural architectures (sparse in the sense that there are fewer neurons than data) are solutions to variational problems over $\M_{\LOp}^k(\R^d)$. In particular, the regularity of functions imposed by the Banach structure of $\M_{\LOp}^k(\R^d)$ explains the variational optimality of the
architectures in \cref{eq:extreme-points-rep-thm}. Furthermore, by the isomorphism in \cref{item:native-space-bona-fide} of \cref{thm:Banach-structure}, we see that $\M_{\LOp}^k(\R^d)$ is a \emph{non-reflexive Banach space} (since $\M_\iso(\Xi_k)$ is non-reflexive), which shows that $\M_{\LOp}^k(\R^d)$ is differs in a fundamental way from a Hilbert space. Said differently, \cref{thm:representer} provides a function-space framework for neural networks that differs in a fundamental way from the (Hilbertian) framework of the neural tangent kernel~\cite{jacot}.

The two extremes of the theorem ($k = 0$ and $k = (d-1)$) capture well-studied problems. Indeed, when $k = 0$, we can take advantage of the fact that the effect of the $k$-plane transform essentially disappears. Indeed, when $k=0$ we have (see \cref{subsubsec:k=0}) that
\begin{equation}
    \norm{\KOp_{d}\RadonOp_0\LOp f}_{\M(\Xi_0)}
    = c_{d,0} \norm{(\mat{A}, \vec{t}) \mapsto \LOp\curly{f}(\mat{A}^\T\vec{t})}_{\M(\Xi_0)}
    = \norm{\LOp f}_{\M(\R^d)}.
\end{equation}
Therefore, the variational problem in \cref{eq:variational-problem} reduces to the well-studied variational problem for $\LOp$-splines~\cite{FisherJerome,UnsergTV}. Since $\LOp$ is isotropic from the admissibility assumptions (\cref{defn:admissible}), the atoms $\rho_{\LOp}$ are radial basis functions (RBFs).

For these problems, the classical theory~\cite{FisherJerome,UnsergTV} suggests that the extreme point solutions are built from atoms of the form $\rho_{\LOp}(\dummy - \vec{\tau}_n)$, $\vec{\tau}_n \in \R^d$, where $\rho_{\LOp}: \R^d \to \R$ is the (canonical) Green's function of $\LOp$ defined in the Fourier domain by $\hat{\rho}_{\LOp} = 1 / \hat{L}$. We can quickly see that \cref{thm:representer} recovers this result since the atoms take the form
\begin{equation}
    \vec{x} \mapsto \rho_{\LOp}(\mat{A}_n\vec{x} - \vec{t}_n)
    = \rho_{\LOp}(\mat{A}_n^\T\mat{A}\vec{x} - \mat{A}_n^\T\vec{t}_n)
    = \rho_{\LOp}(\vec{x} - \mat{A}_n^\T\vec{t}_n)
    = \rho_{\LOp}(\vec{x} - \vec{\tau}_n),
\end{equation}
where we made the substitution $\vec{\tau}_n = \mat{A}_n^\T\vec{t}_n \in \R^d$ in the last equality. Here,
we used the fact that $\mat{A}\mat{A}^\T = \mat{A}^\T\mat{A} = \mat{I}_d$ when $\mat{A} \in V_d(\R^d)$ (i.e., the $k=0$ Stiefel manifold is the space of $(d \times d)$ orthogonal matrices).

At the opposite extreme, when $k = (d-1)$, the atoms take the form
\begin{equation}
    \vec{x} \mapsto \rho_{\LOp}(\vec{\alpha}^\T\vec{x} - t),
\end{equation}
where $\rho_{\LOp}: \R \to \R$ is the Green's function of the univariate operator $\LOp_\rad$, specified by the frequency response $\hat{L}_\rad$ of the univariate radial profile of $\LOp$. These atoms are classical neurons with univariate nonlinearities. This problem was first studied in~\cite{ParhiShallowRepresenter} with $\LOp_\rad = \partial_t^m$, which corresponds to nonlinearities proportional to the truncated power functions $t \mapsto t_+^{m-1}$ (which is the ReLU when $m = 2$), and then generalized to other regularization operators $\LOp$ in~\cite{unser2022kernel}.

\section{Observations and Examples of Compatible Neural Architectures} \label{sec:examples}

Since \cref{thm:representer} guarantees the existence of a solution to \cref{eq:variational-problem} that takes the form in \cref{eq:extreme-points-rep-thm}, we can always find an admissible solution by solving the neural network training problem
\begin{align}
    \min_{\vec{\theta}} \quad & \paren*{\sum_{m=1}^M \Loss(y_m, f_\vec{\theta}(\vec{x}_m)) + \lambda \sum_{n=1}^N \abs{v_n}} \nonumber \\
    \subj\quad & \mat{A}_n\mat{A}_n^\T = \mat{I}_{d-k}, \: n = 1, \ldots, N,
    \label{eq:NN-problem}
\end{align}
for some fixed width $N \geq M$ with
\begin{equation}
    f_\vec{\theta}(\vec{x}) = c(\vec{x}) + \sum_{n=1}^N v_n \, \rho_{\LOp}(\mat{A}_n\vec{x} -
      \vec{t}_n), \quad \vec{x} \in \R^d, \: \vec{\theta} = \curly{v_n, \mat{A}_n, \vec{t}_n}_{n=1}^N \cup \curly{c(\dummy)}.
\end{equation}
The assumption that $N \geq M$ ensures that a solution to the variational problem in \cref{eq:variational-problem} exists in the neural network parameter space (indexed by $\vec{\theta}$) thanks to the bound in \cref{thm:representer}. This assumption implies that, as long as the neural network problem is critically parameterized or overparameterized, its solutions will also be solutions to the variational problem in \cref{thm:representer}. Thus, this result provides insight on the role of overparameterization. We also remark that the constraint on the weight matrices in \cref{eq:NN-problem} corresponds to orthogonal weight normalization. The latter has become a popular architectural choice as it has been shown to increase the stability and improve the generalization performance of neural networks~\cite{anil2019sorting,huang2018orthogonal,huang2023normalization,li2019preventing}.

The nonlinearity $\rho_{\LOp}: \R^{d-k} \to \R^{d-k}$ that appears in \cref{eq:extreme-points-rep-thm} can be viewed as the Green's function of the operator $\LOp_{d-k}: \Sch(\R^{d-k}) \to \Sch'(\R^{d-k})$ which shares the radial profile $\hat{L}_\rad$ of the $k$-plane-admissible operator $\LOp$. That is to say,
\begin{equation}
    \hat{L}_{d-k}(\vec{\omega}) = \hat{L}_\rad(\norm{\vec{\omega}}_2).
\end{equation}
Furthermore, due to the intertwining properties of the $k$-plane transform~\cite[Corollary~7.8]{Parhikplane}, it turns out that the regularization operator in \cref{eq:variational-problem} has the alternative specification
\begin{equation}
    \KOp_{d-k}\RadonOp_k\LOp = \LOp_{d-k}\KOp_{d-k}\RadonOp_k.
\end{equation}

The framework of \cref{thm:representer} encapsulates many neural architectures. The prototypical example of such an operator is the fractional Laplacian $\LOp = (-\Delta)^{\frac{\alpha}{2}}$ and, so, $\LOp_{d-k} = (-\Delta_{d-k})^{\frac{\alpha}{2}}$. The radial profile for this family of operators is
\begin{equation}
    \hat{L}_\rad(\omega) = \abs{\omega}^{\alpha}.
\end{equation}
From \cref{defn:admissible}, the reader can readily verify that this operator is $k$-plane-admissible for $\alpha > (d - k)$. This simple operator encapsulates several known results. At one extreme ($k = (d-1)$), we recover\footnote{Technically, $\rho_{(-\Delta)}(t) = \abs{t}/2$ in this case, but since $t \mapsto \abs{t}/2$ differs from the ReLU $t \mapsto t_+$ by a null space component (affine function), the ReLU and absolute value nonlinearity are treated the same in this framework.} the classical ReLU neurons by the choice $\LOp = (-\Delta)$~\cite{OngieRadon,ParhiShallowRepresenter}. At the opposite extreme ($k = 0$), from \cref{sec:discussion}, we see that we recover a sparse variant of the classical thin-plate/polyharmonic spline RBFs of Duchon~\cite{duchon1977splines} by the choice $\LOp = (-\Delta)^{\frac{\alpha}{2}}$, $\alpha > d$.

In the intermediate regime $1 \leq k \leq (d - 2)$, we can choose $\LOp_{d-k} = (-\Delta_{d-k})^{\frac{1 + (d - k)}{2}}$ so that $\rho_{\LOp}(\vec{t}) \propto \norm{\vec{t}}_2$, $\vec{t} \in \R^{d-k}$, is the norm activation function that has been used for neural architectures in~\cite{gulcehre2014learned}. These observations follow from the fact that the Green's function of the fractional Laplacian $(-\Delta_n)^\frac{\alpha}{2}$ (which acts on $n$-variables) for $\alpha > n$ takes the form
\begin{equation}
    k_{\alpha, n}(\vec{t}) = \FourierOp^{-1}\curly*{\frac{1}{\norm{\dummy}_2}}(\vec{t}) = \begin{cases}
        A_{\alpha, n} \norm{\vec{t}}_2^{\alpha - n}, & \alpha - n \not\in 2\N \\
        B_{m, n} \norm{\vec{t}}_2^{2m} \log \norm{\vec{t}}_2, & \alpha - n = 2m, m \in \N \\
        \Delta_n^{-m}\curly{\delta}, & -\alpha/2 = m, m \in \N,
    \end{cases}
\end{equation}
with $\vec{t} \in \R^n$, $A_{\alpha, n} = \frac{\Gamma((n - \alpha)/2)}{2^\alpha \pi^{n/2} \Gamma(\alpha/2)}$, and $B_{m, n} = \frac{(-1)^{1 + m}}{2^{2m + n - 1} \pi^{n/2} \Gamma(m + n/2)m!}$~\cite{GelfandV1,SamkoBook}. In general, there exist many nonlinearities that are compatible with the presented framework. All that needs to be verified is the admissibility conditions (\cref{defn:admissible}) of the underlying regularization operator.

\section{Connections to RKBS Methods and Variation Spaces} \label{sec:RKBS-variation}
After~\cite{ParhiShallowRepresenter}, a recent line of research has been trying to understand neural networks through the lens of reproducing kernel Banach spaces~\cite{BartolucciRKBS,SpekRKBS}. These works consider Banach spaces defined on, say, $\R^d$ whose members are defined as integral combinations of atoms from some continuously indexed dictionary $\Dict$. The elements of the dictionary are assumed to be continuously indexed by $\xi \in \Xi$, where $\Xi$ is assumed to be some locally compact Hausdorff space. That is, $\mathcal{D} = \curly{\varphi_\xi}_{\xi \in \Xi}$, with the additional hypothesis that $\xi \mapsto \varphi_\xi(\vec{x}) \in C_0(\Xi)$ for any $\vec{x} \in \R^d$.

It turns out that the space
\begin{equation}
    \mathcal{B}(\R^d) \coloneqq \curly*{f: \R^d \to \R \text{ is measurable} \st \text{there exists } \mu \in \M(\Xi) \text{ s.t. } f = \int_{\Xi} \varphi_\xi \dd \mu(\xi)}
    \label{eq:I-RKBS}
\end{equation}
forms a Banach space when equipped with the norm
\begin{equation}
    \norm{f}_{\mathcal{B}} \coloneqq \inf_{\mu \in \M(\Xi)}\: \norm{\mu}_\M \quad\subj\quad f = \int_{\Xi} \varphi_\xi \dd \mu(\xi).
\end{equation}
The assumptions on $\mathcal{D} = \curly{\varphi_\xi}_{\xi \in \Xi}$ ensure that the point evaluation is continuous on $\mathcal{B}(\R^d)$ (i.e., $\delta(\dummy - \vec{x}_0) \in \mathcal{B}'(\R^d)$). Such Banach spaces are referred to as \emph{reproducing kernel Banach spaces}~\cite{lin2022reproducing,zhang2009reproducing}. An RKBS formed from integral combinations of atoms from some continuously indexed dictionary is an \emph{integral RKBS} (I-RKBS)~\cite{SpekRKBS}. With this formalism,~\cite{BartolucciRKBS,SpekRKBS} study many properties of $\mathcal{B}(\R^d)$ as well as data-fitting problems over these spaces with associated representer theorems. We remark that, thanks to the assumption $\xi \mapsto \varphi_\xi(\vec{x}) \in C_0(\Xi)$ for any $\vec{x} \in \R^d$, these authors implicitly ensure that the point evaluation is actually weak$^*$-continuous on $\mathcal{B}(\R^d)$, which is stronger than standard continuity. This property is critical in proving the existence of solutions to data-fitting problems over these spaces.

Our native spaces $\M_{\LOp}^k(\R^d)$ are compatible with the I-RKBS framework. We first note that \cref{item:weak*} in \cref{thm:Banach-structure} ensures that the point evaluation is weak$^*$-continuous and, hence, continuous on $\M_{\LOp}^k(\R^d)$. Thus, $\M_{\LOp}^k(\R^d)$ is an RKBS. Next, we have the direct-sum decomposition
\begin{equation}
    \M_{\LOp}^k(\R^d) = \LOp_{\RadonOp_k}^\dagger\paren*{\M_\iso(\Xi_k)} \oplus \Poly_{n_{\LOp}}(\R^d) = \LOp_{\RadonOp_k}^\dagger\paren*{\M(\Xi_k)} \oplus \Poly_{n_{\LOp}}(\R^d),
\end{equation}
where the first equality is from \cref{item:direct-sum} in \cref{thm:Banach-structure} and the second inequality follows since the null space of $\RadonOp_k^*$ is $\paren{\M_\iso(\Xi_k)}^\comp$ (recall that $\LOp_{\RadonOp_k}^{\dagger} = (\Id - \P_{\Poly_{n_{\LOp}}(\R^d)})\LOp^{-1} \RadonOp_k^*$ and see \cref{prop:Miso-k-plane}). This immediately implies that $\M_{\LOp}^k(\R^d)$ is the direct sum of an I-RKBS with $\Poly_{n_{\LOp}}(\R^d)$, where the dictionary is comprised of the kernels $g_{\mat{A}, \vec{t}}$ from \cref{eq:g-kernel}, continuously indexed by $(\mat{A}, \vec{t}) \in \Xi_k$. This correspondence allows us to directly apply any I-RKBS developments to $\M_{\LOp}^k(\R^d)$.

The study of variation spaces to understand neural networks is a classical endeavor~\cite{kurkova2001bounds,mhaskar2004tractability}. These spaces have received renewed interest~\cite{Bach,devore2023weighted,shenouda2023vector,SXSharp,siegel2023characterization} as a means towards the understanding of the reason why neural networks seem to ``break'' the curse of dimensionality through the lens of nonlinear approximation theory. It turns out that the variation space for a dictionary $\mathcal{D}$ exactly coincides with the I-RKBS so long as the members of $\mathcal{D}$ are sufficiently regular (see~\cite[Lemma~3]{siegel2023characterization}). Indeed, in that case, the variation space for $\Dict$ is the Banach space $(\mathcal{B}(\R^d), \norm{\dummy}_\mathcal{B})$ defined in \cref{eq:I-RKBS}. Thus, $\M_{\LOp}^k(\R^d)$ can also be viewed as a variation space. The investigation of the implications of these tight connections to I-RKBS and variation spaces towards the understanding of neural architectures with multivariate nonlinearities is a direction for future work.

\appendix

\section{Proof of \Cref{thm:right-inverse}} \label[appendix]{app:right-inverse}
\begin{proof}
    We first note that $\LOp_{\RadonOp_k}$ maps $\M_{\LOp}^k(\R^d) \to \M_\iso(\Xi_k)$ by design due to the isotropic symmetry of the $k$-plane domain. The remainder of the proof is divided into three parts.

    \paragraph{Part (i): Existence/Continuity of $\LOp_{\RadonOp_k}^\dagger$ and the Right-Inverse Property \cref{eq:isometry-Miso}}
    We first note that $\LOp^*: \Sch(\R^d) \to L^1_{n_{\LOp}}(\R^d)$ is a continuous injection (\cref{item:L*-mapping} in \cref{defn:admissible}). Therefore, there exists an inverse operator $\LOp^{*-1}: \LOp^*\paren*{\Sch(\R^d)} \to \Sch(\R^d)$ such that $\LOp^{*-1}\LOp^* = \Id$ on $\Sch(\R^d)$. In particular, $\LOp^{*-1}$ is the LSI operator specified by the frequency response $\vec{\xi} \mapsto 1 / \hat{L}(\vec{\xi})$ (since $\LOp^*$ is necessarily self-adjoint from \cref{item:isotropic} in \cref{defn:admissible}).

    Next, we define the operator
    \begin{equation}
        \LOp_{\RadonOp_k}^{\dagger*} \coloneqq \RadonOp_k \LOp^{-1*} (\Id - \P_{\Poly_{n_{\LOp}}(\R^d)}^*),
        \label{eq:L-k-adj}
    \end{equation}
    where
    \begin{equation}
        \P_{\Poly_{n_{\LOp}}(\R^d)}^*\curly{f} = \sum_{\abs{\vec{n}} \leq n_{\LOp}} \ang{m_\vec{n}, f} m_\vec{n}^*,
        \label{eq:dual-proj}
    \end{equation}
    which is the projection of $f$ onto the dual space $\paren*{\Poly_{n_{\LOp}}(\R^d)}' \subset \Sch(\R^d)$. By recalling from \cref{eq:poly} that $m_\vec{n}(\vec{x}) = \vec{x}^\vec{n} / \vec{n}!$, we see that \cref{eq:dual-proj} is a well-defined operator so long as $f$ has sufficient decay (e.g., $f \in \LOp^*\paren*{\Sch(\R^d)} \subset L^1_{n_{\LOp}}(\R^d)$). 
    This reveals that
    \begin{equation}
        \LOp_{\RadonOp_k}^{\dagger*}: \LOp^*\paren[\big]{\Sch(\R^d)} \to \Sch_k.
    \end{equation}
    To check the continuity of this operator, we characterize the boundedness of the (Schwartz) kernel of $\LOp_{\RadonOp_k}^{\dagger*}$. This kernel can be formally identified with $((\mat{A}, \vec{t}), \vec{x}) \mapsto h_\vec{x}(\mat{A}, \vec{t}) \coloneqq \LOp_{\RadonOp_k}^{\dagger*}\curly{\delta(\dummy - \vec{x})}(\mat{A}, \vec{t})$. By the Fourier slice theorem,
    \begin{align}
        \reallywidehat{h_\vec{x}(\mat{A}, \dummy)}(\vec{\omega})
        &= \frac{\FourierOp\curly*{\delta(\dummy - \vec{x}) - \sum_{\abs{\vec{n}} \leq n_{\LOp}} \ang{m_\vec{n}, \delta(\dummy - \vec{x})} m_\vec{n}^*}(\mat{A}^\T\vec{\omega})}{\hat{L}_\rad(\norm{\vec{\omega}}_2)} \nonumber \\
        &= \frac{e^{-\imag \vec{\omega}^\T\mat{A}\vec{x}} - \sum_{\abs{\vec{n}} \leq n_{\LOp}} \frac{\vec{x}^\vec{n}}{\vec{n}!} \, \hat{m}_\vec{n}^*(\mat{A}^\T\vec{\omega})}{\hat{L}_\rad(\norm{\vec{\omega}}_2)} \nonumber \\
        &= \frac{e^{-\imag \vec{\omega}^\T\mat{A}\vec{x}} - \sum_{\abs{\vec{n}} \leq n_{\LOp}} \frac{\vec{x}^\vec{n}}{\vec{n}!} \, (-\imag \mat{A}^\T\vec{\omega})^\vec{n} \, \hat{\kappa}_\rad(\norm{\vec{\omega}}_2)}{\hat{L}_\rad(\norm{\vec{\omega}}_2)} \nonumber \\
        &= \frac{e^{-\imag \vec{\omega}^\T\mat{A}\vec{x}} - \sum_{n=0}^{n_{\LOp}} \frac{(-\imag\vec{\omega}^\T\mat{A}\vec{x})^n}{n!} \, \hat{\kappa}_\rad(\norm{\vec{\omega}}_2)}{\hat{L}_\rad(\norm{\vec{\omega}}_2)},\nonumber \\
        &= \frac{e^{-\imag \vec{\omega}^\T\mat{A}\vec{x}} - \hat{\kappa}_\rad(\norm{\vec{\omega}}_2) \sum_{n=0}^{n_{\LOp}} \frac{(-\imag\vec{\omega}^\T\mat{A}\vec{x})^n}{n!}}{\hat{L}_\rad(\norm{\vec{\omega}}_2)}, \label{eq:numerator}
    \end{align}
    where the penultimate line holds by the multinomial expansion. Note that the quantity in \cref{eq:numerator} is well-defined despite the pole of multiplicity $\gamma_{\LOp}$ at $\vec{\omega} = \vec{0}$. Since $\gamma_{\LOp} \in (n_{\LOp}, n_{\LOp} + 1]$, the form of the numerator ensures a proper pole-zero cancellation with the denominator. Indeed, by Taylor's theorem, when $t \in \R$ is in a neighborhood of $0$, we have that
    \begin{equation}
        e^t - \sum_{n=1}^{n_{\LOp}} \frac{t^n}{n!} = O(t^{n_{\LOp}+1}).
    \end{equation}
    By the identification of the numerator\footnote{This identification is valid by the substitution $t = -\imag \vec{\omega}^\T\mat{A}\vec{x}$ and noting that $\hat{\kappa}_\rad(\norm{\vec{\omega}}_2) = 1$ for $\norm{\vec{\omega}}_2 < R_0$, for some $R_0 \leq 1/2$~\cite[p.~6]{UnserRidges}.} in \cref{eq:numerator} with the above display combined with the property of \cref{item:zero-cancel} in \cref{defn:admissible}, we have that \cref{eq:numerator} is well-defined.
    Next, since 
    $\hat{\kappa}_\rad(\norm{\vec{\omega}}_2) \leq 1$ for $\norm{\vec{\omega}}_2 < 1$ and $\hat{\kappa}_\rad(\norm{\vec{\omega}}_2) = 0$ for $\norm{\vec{\omega}}_2 \geq 1$ (\cref{subsec:poly-proj}), on one hand we have for $\norm{\vec{\omega}}_2 < 1$ that $\abs{\reallywidehat{h_\vec{x}(\mat{A}, \dummy)}(\vec{\omega})}$ is bounded by a constant which depends on $\mat{A}$ and $\vec{x}$. To see the dependence on $\mat{A}$ and $\vec{x}$, we note that
    \begin{align*}
        &\phantom{{}={}}\sum_{n=0}^{n_{\LOp}} \frac{1}{n!} \abs{\imag\vec{\omega}^\T\mat{A}\vec{x}}^n \\
        &\leq \sum_{n=0}^{n_{\LOp}} \frac{1}{n!} \norm{\vec{\omega}}_2^n\norm{\mat{A}\vec{x}}_2^n \\
        &\leq \sum_{n=0}^{n_{\LOp}} \frac{(n_{\LOp} - n)!}{n_{\LOp}!} \frac{n_{\LOp}!}{n!(n_{\LOp} - n)!} \norm{\mat{A}\vec{x}}_2^n \\
        &\leq \sum_{n=0}^{n_{\LOp}} \frac{n_{\LOp}!}{n!(n_{\LOp} - n)!} \norm{\mat{A}\vec{x}}_2^n \\
        &= (1 + \norm{\mat{A}\vec{x}}_2)^{n_{\LOp}}. \numberthis
    \end{align*}
    Thus, there exists a universal constant $C_0 > 0$ such that
    \begin{equation}
        \abs{\reallywidehat{h_\vec{x}(\mat{A}, \dummy)}(\vec{\omega})} \leq C_0 (1 + \norm{\mat{A}\vec{x}}_2)^{n_{\LOp}}, \quad \norm{\vec{\omega}}_2 < 1.
        \label{eq:le-1}
    \end{equation}
    On the other hand, when $\norm{\vec{\omega}}_2 \geq 1$, we have that
    \begin{equation}
        \abs{\reallywidehat{h_\vec{x}(\mat{A}, \dummy)}(\vec{\omega})} \leq \frac{1}{\hat{L}_\rad(\norm{\vec{\omega}}_2)}
        \leq \frac{C_0(1 + \norm{\mat{A}\vec{x}}_2)^{n_{\LOp}}}{\hat{L}_\rad(\norm{\vec{\omega}}_2)}.
        \label{eq:ge-1}
    \end{equation}
    From \cref{item:order-high} in \cref{defn:admissible}, we have that
    \begin{equation}
        \frac{1}{\hat{L}_\rad(\norm{\vec{\omega}}_2)} \leq \frac{\norm{\vec{\omega}}_2^{-\gamma_{\LOp}'}}{C'},
        \label{eq:from-def}
    \end{equation}
    where $\gamma_{\LOp}' > (d - k)$. This property ensures that $\reallywidehat{h_\vec{x}(\mat{A}, \dummy)} \in L^1(\R^{d-k})$. Combining \cref{eq:le-1,eq:ge-1,eq:from-def}, implies that
    \begin{equation}
        (1 + \norm{\mat{A}\vec{x}}_2)^{-n_{\LOp}} \norm{\reallywidehat{h_\vec{x}(\mat{A}, \dummy)}}_{L^1} < \infty.
    \end{equation}
    Since this bound is uniform in $\vec{x} \in \R^d$ and $\mat{A} \in V_{d-k}(\R^d)$ and the inverse Fourier transform $\FourierOp_{d-k}^{-1}: L^1(\R^{d-k}) \to C_0(\R^{d-k})$ is a bounded operator (Riemann--Lebesgue lemma), we have that
    \begin{equation}
        \sup_{\substack{\vec{x} \in \R^d \\ (\mat{A}, \vec{t}) \in \Xi_k}} \abs{h_\vec{x}(\mat{A}, \vec{t})} (1 + \norm{\mat{A}\vec{x}}_2)^{-n_{\LOp}} < \infty
    \end{equation}
    with the property that $h_\vec{x} \in C_{0, \iso}(\Xi_k)$, due to the isotropic symmetry of the $k$-plane transform. Finally, if we write
    \begin{equation}
        \mat{A}\vec{x} = \begin{bmatrix}
            \vec{\alpha}_1^\T\vec{x} \\ \vdots \\ \vec{\alpha}_{d-k}^\T\vec{x}
        \end{bmatrix},
    \end{equation}
    where $\vec{\alpha}_n$ is the $n$th row of $\mat{A}$, then we have that
    \begin{equation}
        \norm{\mat{A}\vec{x}}_2 = \sqrt{\paren{\vec{\alpha}_1^\T\vec{x}}^2 + \cdots + \paren{\vec{\alpha}_{d-k}^\T\vec{x}}^2} \leq 
        \abs{\vec{\alpha}_1^\T\vec{x}} + \cdots + \abs{\vec{\alpha}_{d-k}^\T\vec{x}} \leq (d-k) \norm{\vec{x}}_2.
    \end{equation}
    This results in
    \begin{equation}
        \sup_{\substack{\vec{x} \in \R^d \\ (\mat{A}, \vec{t}) \in \Xi_k}} \abs{h_\vec{x}(\mat{A}, \vec{t})} (1 + \norm{\vec{x}}_2)^{-n_{\LOp}} < \infty.
        \label{eq:h-stability}
    \end{equation}
    This bound implies that the operator is actually well-defined on the larger space $L^1_{n_{\LOp}}(\R^d) \supset \LOp^*\paren*{\Sch(\R^d)}$. In fact, since $h_\vec{x} \in C_{0, \iso}(\Xi_k)$ we have for any $f \in L^1_{n_{\LOp}}(\R^d)$ that
    \begin{equation}
        \LOp_{\RadonOp_k}^{\dagger*}\curly{f} = \int_{\R^d} f(\vec{x}) h_\vec{x}(\dummy) \dd\vec{x} \in C_{0, \iso}(\Xi_k).
    \end{equation}
    This, combined with \cref{eq:h-stability}, ensures that the operator
    \begin{equation}
        \LOp_{\RadonOp_k}^{\dagger*}: (L^1_{n_{\LOp}}(\R^d), \norm{\dummy}_{L^1_{n_{\LOp}}}) \to (C_{0, \iso}(\Xi_k), \norm{\dummy}_{L^\infty})
    \end{equation}
    is continuous and, subsequently, that its adjoint 
    \begin{equation}
        \paren{\LOp_{\RadonOp_k}^{\dagger*}}^* = \LOp_{\RadonOp_k}^{\dagger} = (\Id - \P_{\Poly_{n_{\LOp}}(\R^d)})\LOp^{-1} \RadonOp_k^*: (\M_{\iso}(\Xi_k), \norm{\dummy}_\M) \to (L^\infty_{-n_{\LOp}}(\R^d), \norm{\dummy}_{L^\infty_{-n_{\LOp}}})
        \label{eq:right-inverse-op-defn}
    \end{equation}
    is also continuous.

    We now prove the right-inverse property \cref{eq:isometry-Miso}. Recall that $\LOp^{*-1}\LOp^* = \Id$ on $\Sch(\R^d)$. Thus, by duality, we have the identity
    \begin{equation}
        \LOp\LOp^{-1} = \Id \text{ on } \Sch'(\R^d).
    \end{equation}
    Since $\M_\iso(\Xi_k)$ continuously embeds into $\Sch_k'$ (\cref{prop:Miso-k-plane}), given $u \in \M_\iso(\Xi_k)$, we have that
    \begin{align}
        \LOp_{\RadonOp_k}\LOp_{\RadonOp_k}^\dagger\curly{u}
        &= \KOp_{d-k}\RadonOp_k\LOp (\Id - \P_{\Poly_{n_{\LOp}}(\R^d)}) \LOp^{-1} \RadonOp_k^* \curly{u} \nonumber \\
        &= \KOp_{d-k}\RadonOp_k\LOp\LOp^{-1} \RadonOp^* \curly{u} - \KOp_{d-k}\RadonOp_k\underbrace{\LOp \curly*{\P_{\Poly_{n_{\LOp}}(\R^d)}\curly*{\LOp^{-1} \RadonOp_k^* \curly{u}}}}_{=\,0} \nonumber \\
        &= \KOp_{d-k}\RadonOp_k\RadonOp_k^* \curly{u} \nonumber \\
        &= u,
    \end{align}
    where the last equality follows from \cref{prop:Miso-k-plane}.

    \paragraph{Part (ii): The Pseudo-Left-Inverse Property \cref{eq:comp-proj-inverse}}
    Observe that, for any $f \in L^\infty_{-n_{\LOp}}(\R^d)$,
    \begin{equation}
        \LOp^{-1} \LOp \curly{f} = f + p
    \end{equation}
    for some $p \in \Poly_{n_{\LOp}}(\R^d)$. Therefore,
    given $f \in \M_{\LOp}^k(\R^d) \subset L^\infty_{-n_{\LOp}}(\R^d)$, we have that
    \begin{align}
        \LOp_{\RadonOp_k}^\dagger \LOp_{\RadonOp_k} \curly{f}
        &= (\Id - \P_{\Poly_{n_{\LOp}}(\R^d)})\LOp^{-1} \RadonOp^*\KOp_{d-k}\RadonOp_k\LOp \curly{f} \nonumber \\
        &= (\Id - \P_{\Poly_{n_{\LOp}}(\R^d)})\LOp^{-1} \LOp \curly{f} \nonumber \\
        &= (\Id - \P_{\Poly_{n_{\LOp}}(\R^d)})\curly{f + p} \nonumber \\
        &= f + p - \P_{\Poly_{n_{\LOp}}(\R^d)})\curly{f} - \underbrace{\P_{\Poly_{n_{\LOp}}(\R^d)})\curly{p}}_{=\,p} \nonumber \\
        &= (\Id - \P_{\Poly_{n_{\LOp}}(\R^d)})\curly{f}.
    \end{align}

    \paragraph{Part (iii): The Form \cref{eq:g-kernel}, Stability \cref{eq:stability-bound}, and Continuity \cref{eq:kernel-C0} of the Kernel}
    From \cref{eq:right-inverse-op-defn} we immediately see that the kernel takes the form
    \begin{equation}
        g_{\mat{A}, \vec{t}}(\vec{x}) = \rho_{\LOp}(\mat{A}\vec{x} - \vec{t}) - \sum_{\abs{\vec{n}} \leq n_{\LOp}} \ang{m_\vec{n}^* , \rho_{\LOp}(\mat{A}(\dummy) - \vec{t})} m_\vec{n}(\vec{x}),
    \end{equation}
    where $\ang{m_\vec{n}^* , \rho_{\LOp}(\mat{A}(\dummy) - \vec{t})}$ is well-defined since $m_\vec{n}^* \in \Sch(\R^d)$.  Next, we note that the kernel of $\LOp_{\RadonOp_k}^\dagger$ is the ``transpose'' of the kernel of $\LOp_{\RadonOp_k}^{\dagger*}$. Consequently, we have the equality $g_{\mat{A}, \vec{t}}(\vec{x}) = h_\vec{x}(\mat{A}, \vec{t})$. Therefore, \cref{eq:h-stability} is equivalent to the stability bound
    \begin{equation}
        \sup_{\substack{\vec{x} \in \R^d \\ (\mat{A}, \vec{t}) \in \Xi_k}} \abs{g_{\mat{A}, \vec{t}}(\vec{x})} (1 + \norm{\vec{x}}_2)^{-n_{\LOp}} < \infty.
    \end{equation} 
    Finally, in Part (i) of the proof we showed that $h_\vec{x} \in C_{0, \iso}(\Xi_k)$, which proves \cref{eq:kernel-C0}.
\end{proof}

\section{Proof of \Cref{thm:Banach-structure}} \label[appendix]{app:Banach-structure}

\begin{proof}
    \hfill

    \begin{enumerate}
        \item From \cref{eq:isometry-Miso} in \cref{thm:right-inverse}, we readily deduce that
        \begin{align}
            &\LOp_{\RadonOp_k}^\dagger: \M_\iso(\Xi_k) \to \mathcal{V} \nonumber \\
            &\LOp_{\RadonOp_k}: \mathcal{V} \to \M_\iso(\Xi_k)
        \end{align}
        are continuous bijections. Therefore, if we equip $\mathcal{V}$ with the norm in \cref{eq:U-norm}, $\mathcal{V}$ is isometrically isomorphic to $\M_\iso(\Xi_k)$.

        \item To check that the sum is direct, we must verify that $\mathcal{V} \cap \Poly_{n_{\LOp}}(\R^d) = \curly{0}$. By construction, $\mathcal{V} \subset \M_{\LOp}^k(\R^d)$. Therefore, by \cref{eq:comp-proj-inverse} in \cref{thm:right-inverse},
        \begin{equation}
            \LOp_{\RadonOp_k}^\dagger \LOp_{\RadonOp_k} = \Id - \P_{\Poly_{n_{\LOp}}(\R^d)} \text{ on } \mathcal{V}.
        \end{equation}
        Yet, from \cref{item:iso-Miso},
        \begin{equation}
            \LOp_{\RadonOp_k}^\dagger \LOp_{\RadonOp_k} = \Id \text{ on } \mathcal{V}.
        \end{equation}
        Thus, $(\Id - \P_{\Poly_{n_{\LOp}}(\R^d)})\paren*{\M_{\LOp}^k(\R^d)} = \mathcal{V}$.
        Consequently, $\mathcal{V}$ and $\Poly_{n_{\LOp}}(\R^d)$ are complementary Banach subspaces of $\M_{\LOp}^k(\R^d)$ and so $\mathcal{V} \cap \Poly_{n_{\LOp}}(\R^d) = \curly{0}$.
        
        \item Since $\mathcal{V}$ and $\Poly_{n_{\LOp}}(\R^d)$ are complementary Banach subspaces of $\M_{\LOp}^k(\R^d)$, we can decompose any $f \in \M_{\LOp}^k(\R^d)$ as
        \begin{align}
            f
            &= (\Id - \P_{\Poly_{n_{\LOp}}(\R^d)})\curly{f} + \P_{\Poly_{n_{\LOp}}(\R^d)}\curly{f} \nonumber \\
            &= \LOp_{\RadonOp_k}^\dagger \LOp_{\RadonOp_k}\curly{f} + \P_{\Poly_{n_{\LOp}}(\R^d)}\curly{f} \nonumber \\
            &= \LOp_{\RadonOp_k}^\dagger\curly{u} + p,
        \end{align}
        where the second line follows from \cref{eq:comp-proj-inverse} in \cref{thm:right-inverse}. Therefore, we can equip $\M_{\LOp}^k(\R^d)$ with the composite norm
        \begin{align}
            \norm{f}_{\M_{\LOp}^k}
            &\coloneqq \norm{\LOp_{\RadonOp_k}^\dagger\curly{u}}_\mathcal{V} + \norm{p}_{\Poly_{n_{\LOp}}} \nonumber \\
            &= \norm{\LOp_{\RadonOp_k}^\dagger\LOp_{\RadonOp_k} f}_\mathcal{V} + \norm{\P_{\Poly_{n_{\LOp}}(\R^d)} f}_{\Poly_{n_{\LOp}}} \nonumber \\
            &= \norm{\underbrace{\LOp_{\RadonOp_k}\LOp_{\RadonOp_k}^\dagger}_{=\,\Id}\LOp_{\RadonOp_k} f}_\M + \norm{\P_{\Poly_{n_{\LOp}}(\R^d)} f}_{\Poly_{n_{\LOp}}} \nonumber \\
            &= \norm{\LOp_{\RadonOp_k} f}_\M + \norm{\P_{\Poly_{n_{\LOp}}(\R^d)} f}_{\Poly_{n_{\LOp}}}.
        \end{align}
        This norm is an isometric isomorphism with $\mathcal{V} \times \Poly_{n_{\LOp}}(\R^d)$ by design and thus an isometric isomorphism with $\M_\iso(\Xi_k) \times \Poly_{n_{\LOp}}(\R^d)$ by \cref{item:iso-Miso}.

        \item We first note that, in order to equip a Banach space with a weak$^*$ topology, it must be identifiable as the dual of some primary Banach space.

        Next, notice that
        \begin{equation}
            \ang{m_\vec{n}, f} = \ang{m_\vec{n}, \LOp^* \varphi} = \ang{\LOp m_\vec{n}, \varphi} = \ang{0, \varphi} = 0
        \end{equation}
        for all $f = \LOp^*\varphi \in \LOp^*\paren*{\Sch(\R^d)}$ with $\varphi \in \Sch(\R^d)$ and $\abs{\vec{n}} \leq n_{\LOp}$, where we took advantage of the null space property of $\LOp$ (\cref{rem:L-ann}). Therefore, $\P_{\Poly_{n_{\LOp}}(\R^d)}^*\curly{f} = 0$ for all $f \in \LOp^*\paren*{\Sch(\R^d)}$.
        
        This shows that the operator $\LOp_{\RadonOp_k}^*$ (which, the reader can check, maps $\Sch_k \to \LOp^*\paren*{\Sch(\R^d)}$) is such that
        \begin{align*}
           \LOp_{\RadonOp_k}^{\dagger*} \LOp_{\RadonOp_k}^*\curly{\psi}
           &= \RadonOp_k \LOp^{-1*} (\Id - \P_{\Poly_{n_{\LOp}}(\R^d)}^*) \LOp^* \underbrace{\RadonOp_k^* \KOp_{d-k}\curly{\psi}}_{\in \Sch(\R^d)} \\
           &= \RadonOp_k \LOp^{-1*} \LOp^* \RadonOp_k^* \KOp_{d-k}\curly{\psi} \\
           &= \RadonOp_k \RadonOp_k^* \KOp_{d-k}\curly{\psi} \\
           &= \psi, \numberthis
           \label{eq:left-inverse-predual}
        \end{align*}
        for all $\psi \in \Sch_k$. Thus, $\LOp_{\RadonOp_k}^{\dagger*}$ is a left-inverse of $\LOp_{\RadonOp_k}^{*}$. This implies that the normed space $(\Sch_k, \norm{\dummy}_{L^\infty})$ is (isometrically) isomorphic to the normed space $(\LOp^*\paren*{\Sch(\R^d)}, \norm{\dummy}_\mathcal{U})$ with $\norm{u}_\mathcal{U} \coloneqq \norm{\LOp_{\RadonOp_k}^{\dagger*}\curly{u}}_{L^\infty}$. Recall that $C_{0, \iso}(\Xi_k) = \cl{(\Sch_k, \norm{\dummy}_{L^\infty})}$ from \cref{eq:C0-dense} and define the Banach space $\mathcal{U} \coloneqq \cl{(\LOp^*\paren*{\Sch(\R^d)}, \norm{\dummy}_\mathcal{U})}$. We can now invoke the bounded linear transformation theorem~\cite[Theorem I.7, p. 9]{RSBook} on both $\LOp_{\RadonOp_k}^{\dagger*}$ and $\LOp_{\RadonOp_k}^*$ to find that these operators have the continuous extensions
        \begin{align*}
            &\LOp_{\RadonOp_k}^{\dagger*}: \mathcal{U} \to C_{0, \iso}(\Xi_k) \\
            &\LOp_{\RadonOp_k}^*: C_{0, \iso}(\Xi_k) \to \mathcal{U}, \numberthis \label{eq:predual-isometries}
        \end{align*}
        which establishes that $\mathcal{U}$ and $C_{0, \iso}(\Xi_k)$ are (isometrically) isomorphic Banach spaces.
        
        From \cref{item:native-space-bona-fide}, we know that $\M_{\LOp}^k(\R^d)$ is isometrically isomorphic to $\M_\iso(\Xi_k) \times \Poly_{n_{\LOp}}(\R^d)$. Since $\M_\iso(\Xi_k) = \paren*{C_{0, \iso}(\Xi_k)}'$ (see \cref{eq:Miso-Ciso-dual}) and $\Poly_{n_{\LOp}}(\R^d) = \paren*{\Poly_{n_{\LOp}}(\R^d)}''$ (since $\Poly_{n_{\LOp}}(\R^d)$ is finite-dimensional and hence reflexive), we see that there is a predual of $\M_{\LOp}^k(\R^d)$ that is isometrically isomorphic to $C_{0, \iso}(\Xi_k) \times \paren*{\Poly_{n_{\LOp}}(\R^d)}'$.
         
        From \cref{eq:predual-isometries} and \cref{item:iso-Miso}, we have the diagram in \cref{fig:diagram}.
        \begin{figure}
            \centering
            \[
                \begin{tikzcd}[row sep=3em,column sep=5em]
                \mathcal{V} \arrow[yshift=0.5em]{r}{\LOp_{\RadonOp_k}} & \M_\iso(\Xi_k) \arrow[yshift=-0.5em]{l}{\LOp_{\RadonOp_k}^\dagger} \\
                \mathcal{U}  \arrow[dashed]{u}{\text{dual}} \arrow[yshift=-0.5em, swap]{r}{\LOp_{\RadonOp_k}^{\dagger*}} & C_{0, \iso}(\Xi_k) \arrow[dashed, swap]{u}{\text{dual}} \arrow[yshift=0.5em, swap]{l}{\LOp_{\RadonOp_k}^*}
                \end{tikzcd}
            \]
            \caption{Relationships between the function spaces.}
            \label{fig:diagram}
        \end{figure}
        Therein, we see that $\mathcal{V} = \mathcal{U}'$ and so $\X = \mathcal{U} \oplus \paren*{\Poly_{n_{\LOp}}(\R^d)}'$ is such that $\X' = \M_{\LOp}^k(\R^d)$. To complete the proof, we need to establish that $\delta(\dummy - \vec{x}_0) \in \mathcal{U} \oplus \paren*{\Poly_{n_{\LOp}}(\R^d)}'$~\cite[Theorem IV.20, p. 114]{RSBook}. Clearly, $\delta(\dummy - \vec{x}_0) \in \paren*{\Poly_{n_{\LOp}}(\R^d)}'$. Therefore, we only need to check that $\delta(\dummy - \vec{x}_0) \in \mathcal{U}$.
        This is equivalent to $\LOp_{\RadonOp_k}^{\dagger*}\curly{\delta(\dummy - \vec{x}_0)} \in C_{0, \iso}(\Xi_k)$. Since $\LOp_{\RadonOp_k}^{\dagger*}\curly{\delta(\dummy - \vec{x}_0)}(\mat{A}, \vec{t}) = g_{\mat{A}, \vec{t}}(\vec{x}_0)$, the result follows from \cref{eq:kernel-C0} in \cref{thm:right-inverse}.
    \end{enumerate}
\end{proof}

\bibliographystyle{plain} 
\bibliography{ref}

\begin{thebibliography}{10}

\bibitem{abbe2022merged}
Emmanuel Abbe, Enric~Boix Adsera, and Theodor Misiakiewicz.
\newblock The merged-staircase property: a necessary and nearly sufficient
  condition for {SGD} learning of sparse functions on two-layer neural
  networks.
\newblock In {\em Conference on Learning Theory}, pages 4782--4887. PMLR, 2022.

\bibitem{anil2019sorting}
Cem Anil, James Lucas, and Roger Grosse.
\newblock Sorting out {L}ipschitz function approximation.
\newblock In {\em International Conference on Machine Learning}, pages
  291--301. PMLR, 2019.

\bibitem{AziznejadSparse}
Shayan Aziznejad and Michael Unser.
\newblock Multikernel regression with sparsity constraint.
\newblock {\em SIAM Journal on Mathematics of Data Science}, 3(1):201--224,
  2021.

\bibitem{Bach}
Francis Bach.
\newblock Breaking the curse of dimensionality with convex neural networks.
\newblock {\em Journal of Machine Learning Research}, 18(1):629--681, 2017.

\bibitem{BartolucciRKBS}
Francesca Bartolucci, Ernesto De~Vito, Lorenzo Rosasco, and Stefano Vigogna.
\newblock Understanding neural networks with reproducing kernel {B}anach
  spaces.
\newblock {\em Applied and Computational Harmonic Analysis}, 62:194--236, 2023.

\bibitem{BoyerRepresenter}
Claire Boyer, Antonin Chambolle, Yohann De~Castro, Vincent Duval,
  Fr\'{e}d\'{e}ric de~Gournay, and Pierre Weiss.
\newblock On representer theorems and convex regularization.
\newblock {\em SIAM Journal on Optimization}, 29(2):1260--1281, 2019.

\bibitem{BrediesSparsity}
Kristian Bredies and Marcello Carioni.
\newblock Sparsity of solutions for variational inverse problems with
  finite-dimensional data.
\newblock {\em Calculus of Variations and Partial Differential Equations},
  59(1):Paper No. 14, 26, 2020.

\bibitem{cohen2012capturing}
Albert Cohen, Ingrid Daubechies, Ronald DeVore, Gerard Kerkyacharian, and
  Dominique Picard.
\newblock Capturing ridge functions in high dimensions from point queries.
\newblock {\em Constructive Approximation}, 35:225--243, 2012.

\bibitem{dalalyan2008new}
Arnak~S. Dalalyan, Anatoly Juditsky, and Vladimir Spokoiny.
\newblock A new algorithm for estimating the effective dimension-reduction
  subspace.
\newblock {\em Journal of Machine Learning Research}, 9:1647--1678, 2008.

\bibitem{deBoorSmoothingSplines}
Carl {de Boor} and Robert~E. Lynch.
\newblock On splines and their minimum properties.
\newblock {\em Journal of Mathematics and Mechanics}, 15(6):953--969, 1966.

\bibitem{devore2023weighted}
Ronald DeVore, Robert~D. Nowak, Rahul Parhi, and Jonathan~W. Siegel.
\newblock Weighted variation spaces and approximation by shallow {ReLU}
  networks.
\newblock {\em Applied and Computational Harmonic Analysis}, 74, 2025.

\bibitem{duchon1977splines}
Jean Duchon.
\newblock Splines minimizing rotation-invariant semi-norms in {S}obolev spaces.
\newblock In {\em Constructive Theory of Functions of Several Variables}, pages
  85--100, Berlin, Heidelberg, 1977. Springer Berlin Heidelberg.

\bibitem{FisherJerome}
Stephen~D. Fisher and Joseph~W. Jerome.
\newblock Spline solutions to {$L^1$} extremal problems in one and several
  variables.
\newblock {\em Journal of Approximation Theory}, 13(1):73--83, 1975.

\bibitem{FollandRA}
Gerald~B. Folland.
\newblock {\em Real analysis: Modern techniques and their applications}.
\newblock Pure and Applied Mathematics (New York). John Wiley \& Sons, Inc.,
  New York, second edition, 1999.

\bibitem{fukumizu2004dimensionality}
Kenji Fukumizu, Francis~R. Bach, and Michael~I. Jordan.
\newblock Dimensionality reduction for supervised learning with reproducing
  kernel {H}ilbert spaces.
\newblock {\em Journal of Machine Learning Research}, 5(Jan):73--99, 2004.

\bibitem{GelfandIntegralGeometry}
Izrail~M. Gel'fand, Mark~I. Graev, and Naum~Ya. Vilenkin.
\newblock {\em Generalized functions. {V}ol. 5: {I}ntegral geometry and
  representation theory}.
\newblock Academic Press, New York-London, 1966.
\newblock Translated from Russian by Eugene Saletan.

\bibitem{GelfandV1}
Izrail~M. Gel'fand and Georgiy~E. Shilov.
\newblock {\em Generalized functions. {V}ol. {I}: {P}roperties and operations}.
\newblock Academic Press, New York-London, 1964.
\newblock Translated from Russian by Eugene Saletan.

\bibitem{ghorbani2020neural}
Behrooz Ghorbani, Song Mei, Theodor Misiakiewicz, and Andrea Montanari.
\newblock When do neural networks outperform kernel methods?
\newblock {\em Advances in Neural Information Processing Systems},
  33:14820--14830, 2020.

\bibitem{relu-sparse}
Xavier Glorot, Antoine Bordes, and Yoshua Bengio.
\newblock Deep sparse rectifier neural networks.
\newblock In {\em Proceedings of the Fourteenth International Conference on
  Artificial Intelligence and Statistics}, pages 315--323, 2011.

\bibitem{GonzalezRange}
Fulton~B. Gonzalez.
\newblock On the range of the {R}adon {$d$}-plane transform and its dual.
\newblock {\em Transactions of the American Mathematical Society},
  327(2):601--619, 1991.

\bibitem{goodfellow2013maxout}
Ian Goodfellow, David Warde-Farley, Mehdi Mirza, Aaron Courville, and Yoshua
  Bengio.
\newblock Maxout networks.
\newblock In {\em International Conference on Machine Learning}, pages
  1319--1327. PMLR, 2013.

\bibitem{grohs2016alpha}
Philipp Grohs, Sandra Keiper, Gitta Kutyniok, and Martin Sch{\"a}fer.
\newblock $\alpha$-molecules.
\newblock {\em Applied and Computational Harmonic Analysis}, 41(1):297--336,
  2016.

\bibitem{gulcehre2014learned}
Caglar Gulcehre, Kyunghyun Cho, Razvan Pascanu, and Yoshua Bengio.
\newblock Learned-norm pooling for deep feedforward and recurrent neural
  networks.
\newblock In {\em Machine Learning and Knowledge Discovery in Databases:
  European Conference, ECML PKDD 2014, Nancy, France, September 15-19, 2014.
  Proceedings, Part I 14}, pages 530--546. Springer, 2014.

\bibitem{huang2018orthogonal}
Lei Huang, Xianglong Liu, Bo~Lang, Adams Yu, Yongliang Wang, and Bo~Li.
\newblock Orthogonal weight normalization: Solution to optimization over
  multiple dependent {Stiefel} manifolds in deep neural networks.
\newblock In {\em Proceedings of the AAAI Conference on Artificial
  Intelligence}, volume~32, 2018.

\bibitem{huang2023normalization}
Lei Huang, Jie Qin, Yi~Zhou, Fan Zhu, Li~Liu, and Ling Shao.
\newblock Normalization techniques in training {DNN}s: Methodology, analysis
  and application.
\newblock {\em IEEE Transactions on Pattern Analysis and Machine Intelligence},
  2023.

\bibitem{jacot}
Arthur Jacot, Franck Gabriel, and Cl{\'e}ment Hongler.
\newblock Neural tangent kernel: Convergence and generalization in neural
  networks.
\newblock {\em Advances in Neural Information Processing Systems}, 31, 2018.

\bibitem{KeinertInversion}
Fritz Keinert.
\newblock Inversion of {$k$}-plane transforms and applications in computer
  tomography.
\newblock {\em SIAM Review. A Publication of the Society for Industrial and
  Applied Mathematics}, 31(2):273--298, 1989.

\bibitem{keiper2019approximation}
Sandra Keiper.
\newblock Approximation of generalized ridge functions in high dimensions.
\newblock {\em Journal of Approximation Theory}, 245:101--129, 2019.

\bibitem{WahbaSmoothingSplines3}
George~S. Kimeldorf and Grace Wahba.
\newblock Some results on {T}chebycheffian spline functions.
\newblock {\em Journal of Mathematical Analysis and Applications},
  33(1):82--95, 1971.

\bibitem{KurkovaEstimates}
V{\v{e}}ra Kůrková, Paul~C. Kainen, and Vladik Kreinovich.
\newblock Estimates of the number of hidden units and variation with respect to
  half-spaces.
\newblock {\em Neural Networks}, 10(6):1061--1068, 1997.

\bibitem{kurkova2001bounds}
V{\v{e}}ra Kůrková and Marcello Sanguineti.
\newblock Bounds on rates of variable-basis and neural-network approximation.
\newblock {\em IEEE Transactions on Information Theory}, 47(6):2659--2665,
  2001.

\bibitem{li1991sliced}
Ker-Chau Li.
\newblock Sliced inverse regression for dimension reduction.
\newblock {\em Journal of the American Statistical Association},
  86(414):316--327, 1991.

\bibitem{li2019preventing}
Qiyang Li, Saminul Haque, Cem Anil, James Lucas, Roger~B. Grosse, and
  J{\"o}rn-Henrik Jacobsen.
\newblock Preventing gradient attenuation in {L}ipschitz constrained
  convolutional networks.
\newblock {\em Advances in Neural Information Processing Systems}, 32, 2019.

\bibitem{lin2022reproducing}
Rong~Rong Lin, Hai~Zhang Zhang, and Jun Zhang.
\newblock On reproducing kernel {B}anach spaces: {G}eneric definitions and
  unified framework of constructions.
\newblock {\em Acta Mathematica Sinica, English Series}, 38(8):1459--1483,
  2022.

\bibitem{liu2024learning}
Hao Liu and Wenjing Liao.
\newblock Learning functions varying along a central subspace.
\newblock {\em SIAM Journal on Mathematics of Data Science}, 6(2):343--371,
  2024.

\bibitem{MarkoeAnalyticTomo}
Andrew Markoe.
\newblock {\em Analytic tomography}, volume 106 of {\em Encyclopedia of
  Mathematics and its Applications}.
\newblock Cambridge University Press, Cambridge, 2006.

\bibitem{mhaskar2004tractability}
Hrushikesh~N. Mhaskar.
\newblock On the tractability of multivariate integration and approximation by
  neural networks.
\newblock {\em Journal of Complexity}, 20(4):561--590, 2004.

\bibitem{mhaskar2023approximation}
Hrushikesh~N. Mhaskar.
\newblock Approximation by non-symmetric networks for cross-domain learning.
\newblock {\em arXiv preprint arXiv:2305.03890}, 2023.

\bibitem{MHASKAR1992350}
Hrushikesh~N. Mhaskar and Charles~A. Micchelli.
\newblock Approximation by superposition of sigmoidal and radial basis
  functions.
\newblock {\em Advances in Applied Mathematics}, 13(3):350--373, 1992.

\bibitem{mhaskar1995degree}
Hrushikesh~N. Mhaskar and Charles~A. Micchelli.
\newblock Degree of approximation by neural and translation networks with a
  single hidden layer.
\newblock {\em Advances in Applied Mathematics}, 16(2):151--183, 1995.

\bibitem{OngieRadon}
Greg Ongie, Rebecca Willett, Daniel Soudry, and Nathan Srebro.
\newblock A function space view of bounded norm infinite width {ReLU} nets: The
  multivariate case.
\newblock In {\em Proceedings of the International Conference on Learning
  Representations}, pages 1--24, 2020.

\bibitem{parhi2020role}
Rahul Parhi and Robert~D. Nowak.
\newblock The role of neural network activation functions.
\newblock {\em IEEE Signal Processing Letters}, 27:1779--1783, 2020.

\bibitem{ParhiShallowRepresenter}
Rahul Parhi and Robert~D. Nowak.
\newblock Banach space representer theorems for neural networks and ridge
  splines.
\newblock {\em Journal of Machine Learning Research}, 22:Paper No. 43, 40,
  2021.

\bibitem{ParhiDeepRepresenter}
Rahul Parhi and Robert~D. Nowak.
\newblock What kinds of functions do deep neural networks learn? {I}nsights
  from variational spline theory.
\newblock {\em SIAM Journal on Mathematics of Data Science}, 4(2):464--489,
  2022.

\bibitem{ParhiSPM}
Rahul Parhi and Robert~D. Nowak.
\newblock Deep learning meets sparse regularization: A signal processing
  perspective.
\newblock {\em IEEE Signal Processing Magazine}, 40(6):63--74, 2023.

\bibitem{ParhiMinimax}
Rahul Parhi and Robert~D. Nowak.
\newblock Near-minimax optimal estimation with shallow {ReLU} neural networks.
\newblock {\em IEEE Transactions on Information Theory}, 69(2):1125--1140,
  2023.

\bibitem{Parhikplane}
Rahul Parhi and Michael Unser.
\newblock Distributional extension and invertibility of the {$k$}-plane
  transform and its dual.
\newblock {\em SIAM Journal on Mathematical Analysis}, 56(4):4662--4686, 2024.

\bibitem{parkinson2023linear}
Suzanna Parkinson, Greg Ongie, and Rebecca Willett.
\newblock {ReLU} neural networks with linear layers are biased towards single-
  and multi-index models.
\newblock {\em arXiv preprint arXiv:2305.15598}, 2023.

\bibitem{RSBook}
Michael Reed and Barry Simon.
\newblock {\em Methods of Modern Mathematical Physics {I}: Functional
  analysis}.
\newblock Academic Press, 1972.

\bibitem{rosset2007l1}
Saharon Rosset, Grzegorz Swirszcz, Nathan Srebro, and Ji~Zhu.
\newblock {$\ell_1$} regularization in infinite dimensional feature spaces.
\newblock In {\em 20th Annual Conference on Learning Theory}, pages 544--558.
  Springer, 2007.

\bibitem{RubinInversion}
Boris Rubin.
\newblock Inversion of {$k$}-plane transforms via continuous wavelet
  transforms.
\newblock {\em Journal of Mathematical Analysis and Applications},
  220(1):187--203, 1998.

\bibitem{RudinFA}
Walter Rudin.
\newblock {\em Functional analysis}.
\newblock International Series in Pure and Applied Mathematics. McGraw-Hill,
  Inc., New York, second edition, 1991.

\bibitem{SamkoBook}
Stefan~G. Samko, Anatoly~A. Kilbas, and Oleg~I. Marichev.
\newblock {\em Fractional integrals and derivatives}.
\newblock Gordon and Breach Science Publishers, Yverdon, 1993.
\newblock Theory and applications, Edited and with a foreword by S. M.
  Nikol'skiĭ, Translated from the 1987 Russian original, Revised by the
  authors.

\bibitem{SavareseInfiniteWidth}
Pedro Savarese, Itay Evron, Daniel Soudry, and Nathan Srebro.
\newblock How do infinite width bounded norm networks look in function space?
\newblock In {\em Conference on Learning Theory}, pages 2667--2690, 2019.

\bibitem{ScholkopfKernels}
Bernhard Sch{\"o}lkopf and Alexander~J. Smola.
\newblock {\em Learning with Kernels: Support Vector Machines, Regularization,
  Optimization, and Beyond}.
\newblock Adaptive computation and machine learning. MIT Press, 2002.

\bibitem{shenouda2023vector}
Joseph Shenouda, Rahul Parhi, Kangwook Lee, and Robert~D Nowak.
\newblock Variation spaces for multi-output neural networks: Insights on
  multi-task learning and network compression.
\newblock {\em Journal of Machine Learning Research}, 25(231):1--40, 2024.

\bibitem{siegel2023characterization}
Jonathan~W. Siegel and Jinchao Xu.
\newblock Characterization of the variation spaces corresponding to shallow
  neural networks.
\newblock {\em Constructive Approximation}, pages 1--24, 2023.

\bibitem{SXSharp}
Jonathan~W. Siegel and Jinchao Xu.
\newblock Sharp bounds on the approximation rates, metric entropy, and
  {$n$}-widths of shallow neural networks.
\newblock {\em Foundations of Computational Mathematics}, 24(2):481--537, 2024.

\bibitem{SmithRadiographs}
Kennan~T. Smith, Donald~C. Solmon, and Sheldon~L. Wagner.
\newblock Practical and mathematical aspects of the problem of reconstructing
  objects from radiographs.
\newblock {\em Bulletin of the American Mathematical Society},
  83(6):1227--1270, 1977.

\bibitem{SolmonXRay}
Donald~C. Solmon.
\newblock The {$X$}-ray transform.
\newblock {\em Journal of Mathematical Analysis and Applications},
  56(1):61--83, 1976.

\bibitem{sonoda2024unified}
Sho Sonoda, Isao Ishikawa, and Masahiro Ikeda.
\newblock A unified {F}ourier slice method to derive ridgelet transform for a
  variety of depth-2 neural networks.
\newblock {\em Journal of Statistical Planning and Inference}, 233:106184,
  2024.

\bibitem{SpekRKBS}
Len Spek, Tjeerd~Jan Heeringa, and Christoph Brune.
\newblock Duality for neural networks through reproducing kernel {B}anach
  spaces.
\newblock {\em arXiv preprint arXiv:2211.05020}, 2022.

\bibitem{steinwart2003sparseness}
Ingo Steinwart.
\newblock Sparseness of support vector machines.
\newblock {\em Journal of Machine Learning Research}, 4(Nov):1071--1105, 2003.

\bibitem{TrevesTVS}
Fran\c{c}ois Tr\`eves.
\newblock {\em Topological vector spaces, distributions and kernels}.
\newblock Academic Press, New York-London, 1967.

\bibitem{UnserUnifyingRepresenter}
Michael Unser.
\newblock A unifying representer theorem for inverse problems and machine
  learning.
\newblock {\em Foundations of Computational Mathematics}, 21(4):941--960, 2021.

\bibitem{unser2022kernel}
Michael Unser.
\newblock From kernel methods to neural networks: A unifying variational
  formulation.
\newblock {\em Foundations of Computational Mathematics}, 2023.

\bibitem{UnserRidges}
Michael Unser.
\newblock Ridges, neural networks, and the {R}adon transform.
\newblock {\em Journal of Machine Learning Research}, 24:Paper No. 37, 33,
  2023.

\bibitem{UnserDirectSums}
Michael Unser and Shayan Aziznejad.
\newblock Convex optimization in sums of {B}anach spaces.
\newblock {\em Applied and Computational Harmonic Analysis}, 56:1--25, 2022.

\bibitem{UnsergTV}
Michael Unser, Julien Fageot, and John~Paul Ward.
\newblock Splines are universal solutions of linear inverse problems with
  generalized {TV} regularization.
\newblock {\em {SIAM} Review}, 59(4):769--793, 2017.

\bibitem{wahba1990spline}
Grace Wahba.
\newblock {\em Spline models for observational data}.
\newblock SIAM, 1990.

\bibitem{WendlandBook}
Holger Wendland.
\newblock {\em Scattered data approximation}, volume~17 of {\em Cambridge
  Monographs on Applied and Computational Mathematics}.
\newblock Cambridge University Press, Cambridge, 2005.

\bibitem{zhang2009reproducing}
Haizhang Zhang, Yuesheng Xu, and Jun Zhang.
\newblock Reproducing kernel {B}anach spaces for machine learning.
\newblock {\em Journal of Machine Learning Research}, 10(12), 2009.

\end{thebibliography}

\end{document}